\documentclass[a4paper,fleqn]{cas-sc}

\usepackage[authoryear]{natbib}
\usepackage{adjustbox}
\usepackage{stfloats}
\usepackage{subcaption}

\usepackage[capitalize,noabbrev]{cleveref}
\crefname{figure}{Fig.}{Figs.}
\crefname{equation}{Eq.}{Eqs.}
\usepackage{algorithmic}
\usepackage{algorithm}
\usepackage{amsthm}
\theoremstyle{plain}
\newtheorem{theorem}{Theorem}[section]

\newtheorem{lemma}{Lemma}[section]

\theoremstyle{definition}
\newtheorem{definition}{Definition}[section]

\theoremstyle{remark}
\newtheorem{remark}{Remark}[section]

\begin{document}
\let\WriteBookmarks\relax

\title [mode = title]{Time-varying Interaction Graph ODE for Dynamic Graph Representation Learning}                      
%
%

\author[1]{Xiaoyi Wang}
\credit{Conceptualization of this study, Formal analysis, Methodology, Resources, Software, Data curation, Validation, Writing - Original draft preparation}
\author[1]{Zhiqiang Wang}
\cormark[1]
\ead{wangzq@sxu.edu.cn}
\credit{Conceptualization of this study, Funding acquisition, Methodology, Project administration, Supervision, Writing – review and editing}
\author[1]{Jianqing Liang}
\credit{Conceptualization of this study, Funding acquisition, Methodology, Project administration, Supervision, Writing – review and editing}
\author[1]{Xingwang Zhao}
\credit{Methodology, Supervision, Writing – review and editing}
\author[2]{Chuangyin Dang}
\credit{Methodology, Supervision, Writing – review and editing}
\author[3]{Zhen Jin}
\credit{Methodology, Supervision, Writing – review and editing}
\author[1]{Jiye Liang}
\credit{Funding acquisition, Writing – review and editing, Project administration, Supervision}
%
%
\address[1]{Key Laboratory of Computational Intelligence and Chinese Information Processing of Ministry of Education, School of Computer and Information Technology, Shanxi University, Taiyuan, Shanxi 030006, China.}
\address[2]{Department of System Engineering and Engineering Management, City University of Hong Kong, Kowloon 999077, Hong Kong, China.}
\address[3]{Key Laboratory of Complex Systems and Data Science of Ministry of Education, Complex Systems Research Center, Shanxi University, Taiyuan 030006, Shanxi, China.}
%
%
%

%
%
%
%

\cortext[cor1]{Corresponding author}

\nonumnote{This manuscript is under review by Elsevier.}

\begin{abstract}
Graph neural Ordinary Differential Equations (ODE) combine neural ODE with the message passing mechanism of Graph Neural Networks (GNN), providing a continuous-time modeling method for graph representation learning. However, in dynamic graph scenarios, existing graph neural ODEs typically employ a unified message passing mechanism, assuming that inter-node interactions share the same message passing function at any time, which makes it challenging to capture the diversity and time-varying nature of inter-node interaction patterns. To address this, we propose Time-varying Interaction Graph Ordinary Differential Equations (TI-ODE). The core idea of TI-ODE is to decompose the evolution function of a graph ODE into a set of learnable interaction basis functions, where each basis function corresponds to a distinct type of inter-node interaction. These basis functions are dynamically combined through time-dependent learnable weights, enabling inter-node interaction patterns to adaptively evolve over time. Experimental results on six dynamic graph datasets demonstrate that TI-ODE consistently outperforms existing methods and achieves state-of-the-art performance on attribute prediction tasks, and experiments on the \textit{Covid} dataset further verify the interpretability and generalizability of our TI-ODE. Furthermore, we demonstrate both theoretically and empirically that TI-ODE exhibits superior robustness compared to models utilizing a unified message-passing mechanism.
\end{abstract}

%

\begin{keywords}
Graph Neural Networks \sep Dynamic Graph \sep Neural Ordinary Differential Equations \sep Deep Learning
\end{keywords}

\maketitle

\section{Introduction}
Dynamic graphs are ubiquitous in real-world systems such as social networks, transportation networks, and biological networks, where nodes and edges evolve over time and exhibit complex dynamic behaviors. To model such network dynamics, Dynamic Graph Neural Networks (DGNN) \citep{yang2024dynamic,zheng2025survey} have been proposed. Building upon Graph Neural Networks (GNN) \citep{scarselli2008graph}, DGNN explicitly incorporate the temporal dimension to characterize the dynamic evolution of both node states and graph structures. However, traditional DGNN based on discrete neural architectures \citep{zheng2025survey} rely on predefined time steps and fixed hierarchical update mechanisms, which limits their ability to accurately capture continuous-time dynamics. Furthermore, such fixed data-processing pipelines reduce model adaptability, making it difficult to handle irregularly observed data. The redundant structures and parameters inherent in discrete-time designs also lead to increased computational and memory overhead \citep{JFYZ202408014}. To address these limitations, neural Ordinary Differential Equations (ODE) \citep{chen2018neural} offer a continuous-time modeling paradigm that characterizes the system's continuous evolution by parameterizing the derivative of states with respect to time. Researchers have further extended this continuous modeling capability to the domain of dynamic graphs by integrating the message passing mechanism of GNN into the evolution equations of neural ODE \citep{liu2025graph}.

Despite these advances, real-world inter-node interaction patterns often exhibit significant diversity and time-varying nature. Specifically, diversity refers to the fact that inter-node interactions can take multiple types, and multiple different types of interactions may occur simultaneously at the same time; these interactions are not mutually exclusive but can coexist in parallel. Time-varying nature refers to the dynamic evolution of inter-node interaction patterns over time, manifested in the emergence and disappearance of interactions, switching among interaction types, and continuous changes in interaction strength. For example, in social networks, user interactions are diverse (e.g., likes, comments, shares) and show significant dynamic evolution characteristics \citep{thorgren2024temporal}; in infectious disease transmission networks, viruses exhibit multiple transmission modes (e.g., sexual transmission, bloodborne transmission, airborne transmission \citep{shaw2012hiv}), which present significant time-dependent features across different stages \citep{karia2020covid,xin2021government,wang2025multi}. Existing graph neural ODE models \citep{JFYZ202408014,liu2025graph} typically rely on a unified message passing mechanism, which assumes that all nodes share the same message passing function. While this assumption simplifies computation to some extent, a unified message passing function fails to capture the diversity and time-varying nature of inter-node interaction patterns, thereby limiting the model’s expressive power and interpretability \citep{11271751}. Consequently, how to effectively model the diversity and time-varying nature of inter-node interaction patterns has emerged as a key challenge in dynamic graph learning.

To address this challenge, we propose Time-varying Interaction Graph Ordinary Differential Equations (TI-ODE). The core idea of this model is to represent the evolution function in the graph ODE as a set of learnable interaction basis functions, each corresponding to a specific type of inter-node interaction, which are dynamically combined with time-dependent learnable weights. Importantly, the proposed formulation departs from conventional attention or reweighting schemes by representing interaction dynamics through a functional basis expansion, instead of relying on a single shared function to parameterize all pairwise relationships. The main contributions of this work are summarized as follows:
\begin{itemize}
	\item  We propose a graph neural ODE model, TI-ODE, which models inter-node interactions in the graph ODE evolution function as a time-dependent combination of multiple basis functions, capturing both the diversity of interaction patterns and their time-varying nature. 
	\item We theoretically prove that TI-ODE exhibits higher robustness compared to models that adopt a unified message passing function. 
	\item Experimental results demonstrate that TI-ODE consistently outperforms existing state-of-the-art models across six diverse benchmarks, including two physical dynamics datasets, two molecular dynamics datasets, and two real-world datasets. To further validate its practical utility, we evaluate TI-ODE on the \textit{Covid} dataset, demonstrating its superior interpretability and generalizability. Furthermore, additional empirical evaluations indicate that TI-ODE exhibits significantly enhanced robustness compared to models that rely on a unified message-passing function.
\end{itemize}

\section{Related Work}
\subsection{GNN for Dynamic Graph Representation}
Owing to the inherent advantages of GNN in modeling complex structural relations \citep{scarselli2008graph}, they have been widely adopted in dynamic graph representation learning in recent years  \citep{zheng2025survey}. Early studies largely followed the paradigm of static GNN by stacking recurrent architectures over discrete time to capture temporal dependencies  \citep{zheng2025survey}, including DCRNN \citep{li2018diffusion}, GraphWaveNet \citep{wu2019graph}, EvolveGCN \citep{pareja2020evolvegcn}, WD-GCN \citep{manessi2020dynamic}, AGCRN \citep{bai2020adaptive},  TTS-AMP \citep{cini2023taming} and T\&S-AMP \citep{cini2023taming}  . However, discrete-time models typically rely on a fixed sampling frequency, making them ill-suited for irregularly sampled or event-driven dynamic graphs. Consequently, continuous-time dynamic graph models have attracted increasing attention \citep{zheng2025survey}, such as DyREP \citep{trivedi2019dyrep}, JODIE \citep{kumar2019predicting}, TGAT \citep{xu2020inductive}, and TREND \citep{wen2022trend}. Although these methods support continuous-time timestamps, they often still depend on predefined time discretization and fixed update mechanisms, which can lead to information loss and temporal blurring \citep{JFYZ202408014}. In contrast, we leverage neural ODE \citep{chen2018neural} to model node-state evolution as a continuous-time dynamical system, enabling dynamic graphs to evolve naturally in continuous time and alleviating the inherent limitations of discrete-time formulations.
\subsection{Graph Neural ODE}
Graph neural ODE model the temporal evolution of graph-structured data as a continuous-time dynamical system, allowing node states (or topologies) to change smoothly over time in a differentiable manner \citep{liu2025graph}. This approach is inspired by neural ODE \citep{chen2018neural}, which use learnable differential equations to describe the continuous evolution of latent states, naturally adapting to irregular sampling and supporting interpolation and extrapolation at arbitrary time points. Later, Latent-ODE \citep{rubanova2019latent} extended this idea by introducing it into a variational inference framework, effectively handling sparse observations and asynchronous sequences, laying the groundwork for continuous-time modeling of dynamic graphs. In this context, the core idea of graph neural ODE is to incorporate the message passing mechanism of GNN into the ODE evolution function, where the derivatives of node states are jointly determined by neighborhood information flow and the node's own dynamics \citep{liu2025graph}, forming a continuous vector field defined on the graph. Representative methods include LG-ODE \citep{huang2020learning}, CG-ODE \citep{huang2021coupled}, SocialODE \citep{wen2022social}, HOPE \citep{luo2023hope}, TREAT\citep{huang2024physics}, CTAN \citep{gravina2024long}, PGODE \citep{luopgode} and CSG-ODE \citep{wangcsg}. However, existing graph neural ODE typically assume that inter-node interactions share a unified message-passing function, making it difficult to capture the diversity and time-varying nature of real-world interactions. To address this, we introduce a set of learnable interaction basis functions into graph neural ODE, allowing for a more flexible representation of inter-node interactions through a dynamic combination of functions. Importantly, this formulation goes beyond conventional attention or reweighting mechanisms by explicitly modeling interaction dynamics as a functional basis expansion, rather than parameterizing pairwise interactions with a single shared function.

\section{Preliminary}
\subsection{Problem Definition}
Given a sequence of dynamic graph snapshots ${{\cal G}^{1:T}} = \{ {G^1},{G^2}, \ldots ,{G^T}\}$, where $T$ denotes the total number of time steps. For any timestamp $t \in \{ 1,...,T\} $, the state of the dynamic graph is defined as ${G^t} = (V,{E^t},{X^t})$, where $V$ is the node set containing $N$ nodes, ${E^t}$ denotes the edge set at time $t$, and ${X^t} = \{ x_i^t|i \in V\}  \in {\mathbb{R}^{N \times d}}$ is the node feature matrix at time $t$ with feature dimension $d$. In addition, the graph topology can be represented by a weighted adjacency matrix ${A^t} \in {\mathbb{R}^{N \times N}}$, whose entry $A_{ij}^t$ characterizes the connectivity between nodes $i$ and $j$ at timestamp $t$, as well as the corresponding interaction strength. Given a historical dynamic graph sequence ${{\cal G}^{1:{T_h}}} = \{ {G^1},{G^2}, \ldots ,{G^{{T_h}}}\} $, our goal is to learn a parameterized model ${f_\theta }$ that leverages the historical information ${{\cal G}^{1:{T_h}}}$ to predict future node attributes, i.e., to generate ${X^{{T_h}+1 :T}}$.

\subsection{Neural ODE}
Unlike traditional Residual Networks \citep{he2016deep} that rely on stacking discrete layers, neural ODE \citep{chen2018neural} parameterize the derivative of the hidden state using a neural network, thereby transforming a discrete sequence of layers into a continuous transformation process. This approach essentially generalizes the discrete hierarchical structures of deep learning to a continuous limit. Instead of defining a mapping with a finite number of layers, it specifies a vector field that characterizes the continuous evolution of states over time. This process is described by an ODE parameterized by a neural network:
\begin{equation}
	\frac{{dh(t)}}{{dt}} = {f_\theta }(h(t),t),
\end{equation}
where $f$ represents the neural network with parameters $\theta$. The input is first mapped to an initial state ${h}(0)$ and then integrated over the time interval $[0, T]$ using standard ODE numerical methods such as Euler or Runge-Kutta methods to obtain the evolved features. A primary challenge in training such models lies in the efficient computation of gradients. Directly backpropagating through every operation of an ODE solver leads to memory consumption that increases sharply with the number of steps. To mitigate this, neural ODE employ the Adjoint Sensitivity Method \citep{chen2018neural}. This technique computes the gradient of the loss function with respect to model parameters by solving an augmented ODE backward in time rather than storing all intermediate states from the forward pass. The key advantage is a constant memory cost; the memory required for backpropagation remains fixed regardless of the number of integration steps. This not only alleviates the memory bottleneck associated with training deep models but also ensures precise gradient computation while maintaining control over numerical errors.
\begin{figure*}[htbp]
	\centering
	\includegraphics[width=\textwidth]{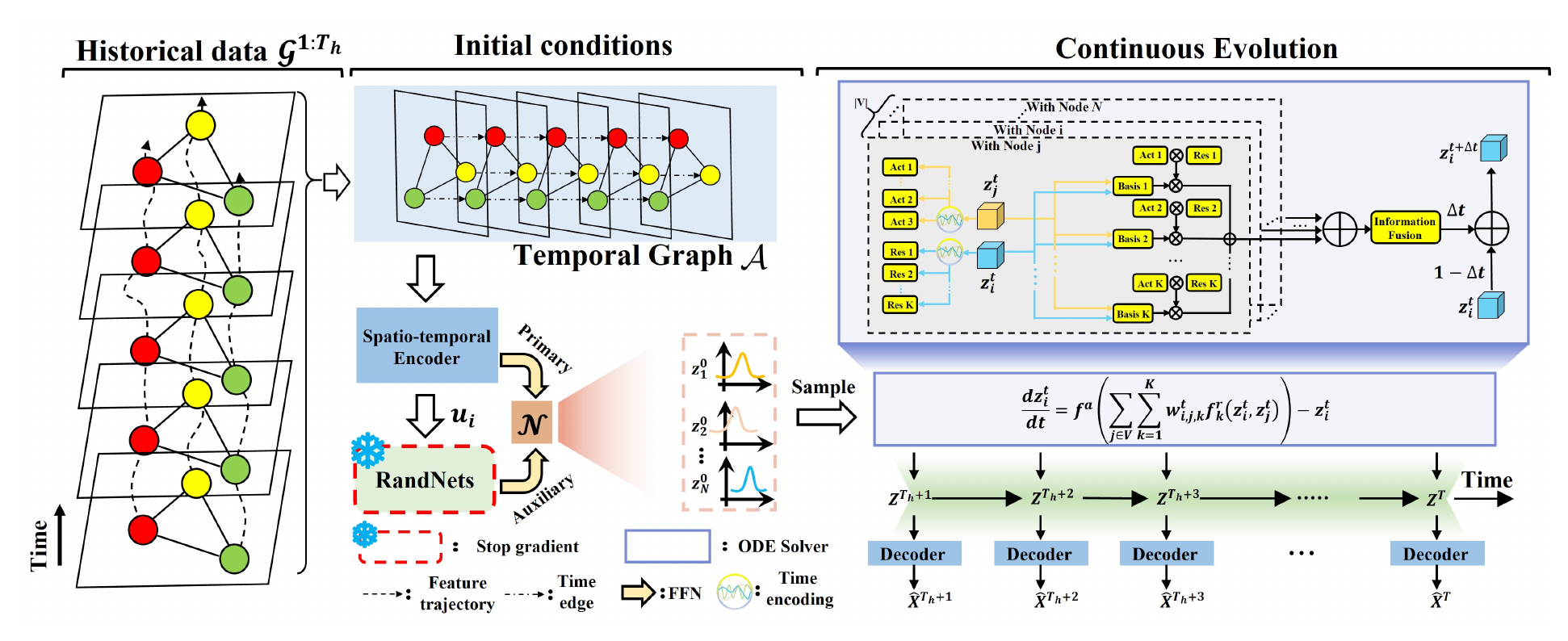}
	\caption{An overview of the proposed TI-ODE. Our TI-ODE first constructs temporal graphs from historical dynamic graph sequences and employs an attention-driven spatio-temporal encoder to obtain node sequence representations through message passing and temporal aggregation, thereby inferring the primary posterior distribution. We then introduce a random mapping network to generate multiple auxiliary posterior distributions, which are adaptively fused with the primary posterior distribution to yield a refined posterior distribution for sampling the initial latent states. Finally, the sampled initial state is propagated forward by TI-ODE in continuous time, and a decoder maps the evolved latent states back to the observation space to predict future node trajectories.}
	\label{overivew}
	
	\vspace{-0.5cm}
\end{figure*}
\section{Methodology}
In this section, we present the design details of TI-ODE, as shown in \textcolor{blue}{\cref{overivew}}. TI-ODE is optimized end-to-end within the Variational AutoEncoder (VAE) framework \citep{kingma2013auto}. We describe the model from three aspects: initial condition construction, time-varying interaction graph ODE, and model optimization.

\subsection{Initial Condition Construction}
\label{Initial Condition Construction}
In graph neural ODE methods, the choice of initial conditions plays a decisive role in learning dynamical trajectories, as it directly affects the model’s ability to capture the underlying evolution laws of the system \citep{liu2025graph}. To enable node representations to preserve key features related to node interactions and dynamic evolution while integrating historical information, we adopt an attention-based spatio-temporal encoder \citep{huang2021coupled}, which progressively aggregates each node’s historical information via message passing on a temporal graph and employs an attention mechanism to fuse features across different time steps through weighted integration. 

We first construct a temporal graph from the historical dynamic graph sequence ${{\cal G}^{1:{T_h}}}$. In this temporal graph, each node ${i^t}$ instance at a timestamp $t$ is treated as a distinct node. Two types of edges are included: (1) temporal edges, which capture the continuous evolution of the same node over time; and (2) spatial edges, which are induced by the dynamic graph snapshot ${G^t}$ at each timestamp $t$ and encode the inter-node topology. Formally, the adjacency matrix ${\cal A}$ of the temporal graph is defined as follows:
\begin{equation}
	\label{st-graph}
	{\cal A}({i^t},{j^{\tilde t}}) = \left\{ \begin{array}{l}
		A_{ij}^t,t = \tilde t,\\
		1,i = j,\tilde t = t + 1,\\
		0,otherwise.
	\end{array} \right.
\end{equation}

Next, we employ an attention-based spatio-temporal encoder to aggregate features from each node’s neighborhood in the temporal graph, thereby progressively integrating historical information.  Let $h_{{i^t}}^{(l)} \in {\mathbb{R}^{{d_{hid}}}}$ denote the hidden representation of node ${i^t}$ at the $l$-th layer, where $d_{hid}$ is the hidden dimension. Since the temporal graph contains both temporal and spatial edges, we introduce a time embedding before message passing to differentiate edge types, i.e., $\hat h_{{i^t}}^{(l)} = h_{{i^t}}^{(l)} + T{E_{(\tilde t - t)}}$,  where $T{E_{(\Delta t,2i)}} = \sin (\frac{{\Delta t}}{{{{10000}^{2i/{d_{hid}}}}}})$ and $T{E_{(\Delta t,2i + 1)}} = \cos (\frac{{\Delta t}}{{{{10000}^{2i/{d_{hid}}}}}})$. The attention weights are computed via query–key matching as follows:
\begin{equation}
	\label{attention}
	\alpha _{j\tilde {^t},{i^t}}^{(l)} = \frac{{{\cal A}({j^{\tilde t}},{i^t})}}{{\sqrt {{d_{hid}}} }}{({W_K}\hat h_{{j^{\tilde t}}}^{(l - 1)})^ \top }({W_Q}h_{{i^t}}^{(l - 1)}),
\end{equation}
where ${W_Q} \in {\mathbb{R}^{{d_{hid}} \times {d_{hid}}}}$ and ${W_K} \in {\mathbb{R}^{{d_{hid}} \times {d_{hid}}}}$ are the learnable weight matrices that project the input features into the query and key vectors, respectively. Next, the node representations are updated by weighted aggregation of neighborhood information:
\begin{equation}
	\label{update}
	h_{{i^t}}^{(l + 1)} = h_{{i^t}}^{(l)} + \sigma (\sum\limits_{j\widetilde {^t} \in {{\cal N}_{{i^t}}}} {\alpha _{j\tilde {^t},{i^t}}^{(l)}{W_V}\hat h_{{j^{\tilde t}}}^{(l)}} ),
\end{equation}
where ${W_V} \in {\mathbb{R}^{{d_{hid}} \times {d_{hid}}}}$ is the learnable weight matrix that maps the input features to value vectors, ${{\cal N}_{{i^t}}}$ denotes the set of all neighboring nodes of node ${i^t}$ in the temporal graph, and $\sigma (\cdot)$ is an activation function. We take the output of the $L$-th layer as the final representation of each node, denoted as ${h_{{i^t}}} = h_{{i^t}}^{(L)}$. Based on this representation, we further construct the corresponding sequential representation for each node, denoted as $\{ {u_i}\} _{i = 1}^N$:
\begin{equation}
	\label{u_attention1}
	{a_{{i^t}}^s = \tanh ((\frac{1}{{{T_h}}}\sum\limits_t^{{T_h}} {{{\hat h}_{{i^t}}}} ){W_a}){{\hat h}_{{i^t}}},}
\end{equation}
\begin{equation}
	\label{u_attention2}
	{u_i} = \frac{1}{{{T_h}}}\sum\limits_t^{{T_h}} {\sigma (a_{{i^t}}^s{{\hat h}_{{i^t}}})} ,
\end{equation}
where ${\widehat h_{{i^t}}} = {h_{{i^t}}} + T{E_{(t)}}$. The attention score $a_{{i^t}}^s$ is computed by matching a global query vector with the final node representation at each timestamp. Concretely, the global query vector is obtained by average pooling over all node representations and then projecting the pooled vector with a learnable matrix ${W_a}$.  Finally, this sequential representation is used to infer the mean and variance of the primary posterior distribution: 
\begin{equation}
	\label{original posterior}
	q(z_i^{0}|{{\cal G}^{1:{T_{h}}}}) = {\cal N}({\phi ^m}({u_i}),{\phi ^v}({u_i})) = {\cal N}({\mu _i},{\sigma _i}),
\end{equation}
where ${\phi ^m}$ and ${\phi ^v}$ are learnable feed-forward neural networks (FNN) used to generate the mean and variance, respectively. To enhance the diversity of the initial latent states, we introduce a set of untrained Random FNNs (RandNet) to stochastically perturb the mean and variance of the primary posterior distribution, thereby expanding the representational space of the posterior distribution. This design is motivated by empirical evidence that randomly initialized, untrained neural networks can yield effective representations without task-specific training  \citep{thompson2022evaluation}. Specifically, we construct multiple RandNets $\phi _k^r$  with random initialization and fixed parameters, which apply additional transformations to the encoder-generated sequential representations, thereby producing a set of auxiliary posterior distributions. We then adaptively fuse the means and variances of the primary posterior distribution with those of the auxiliary posterior distributions to obtain a refined approximate posterior distribution $q(z_i^{0}|{{\cal G}^{1:{T_h}}}) = {\cal N}({\hat \mu _i},{\hat \sigma _i})$:
\begin{align}
	\label{aul posterior1}
	u_{i,k}^r = \phi _k^r({u_i}),
\\
	\label{aul posterior2}
	\mu _{i,k}^r = {\phi ^m}(u_{i,k}^r),\sigma _{i,k}^r = {\phi ^v}(u_{i,k}^r),
\\
	\label{aul posterior3}
	a_{i,k}^r = {\mathop{\rm Tanh}\nolimits} ({\varphi ^a}[{u_i}||u_{i,k}^r]),
\\
	\label{aul posterior4}
	{\hat \mu _i} = {\mu _i} + \sum\limits_{k = 1}^{{K^r}} {a_{i,k}^r\mu _{i,k}^r} ,{\hat \sigma _i} = {\sigma _i} + \sum\limits_{k = 1}^{{K^r}} {a_{i,k}^r\sigma _{i,k}^r} ,
\end{align}
where $u_{i,k}^r$ denotes the additional transformation induced by the $k$-th RandNet, and $\mu _{i,k}^r$ and  $\sigma _{i,k}^r$ are the mean and variance of the corresponding auxiliary posterior distribution, respectively. Let ${K^r}$ denote the number of RandNets, and ${\varphi ^a}$ be a learnable FNN used to compute the fusion weights $a_{i,k}^r$. This design introduces additional random projection views without altering the main architecture of the encoder, thereby providing a more diverse set of candidate initial conditions for subsequent continuous-time dynamical modeling. It is important to emphasize that all parameters of RandNets are kept fixed during training and are not updated by gradient descent, and thus do not introduce additional optimization complexity. The only learnable components are the parameters associated with the fusion weights, enabling the model to adaptively regulate the contributions of different auxiliary posterior distributions. When the auxiliary posterior distributions induced by certain random mappings more effectively complement the information in the primary posterior distribution or better characterize the latent state distribution, their corresponding weights will be increased; otherwise, they will be decreased. Finally, we sample the initial latent state $z_i^0 \sim p(z_i^0) \approx {\cal N}({\hat \mu _i},{\hat \sigma _i})$ for each node from the refined posterior distribution.

\subsection{Time-Varying Interaction Graph ODE}
To model the diversity and time-varying nature of inter-node interaction patterns, we propose a graph neural ODE based on continuous-time interaction basis functions. The core idea is to represent the evolution function of a graph ODE as a set of learnable interaction basis functions, where each basis function corresponds to a distinct type of inter-node interaction, and to dynamically combine them using time-dependent learnable weights. Concretely, the ODE takes the following form:
\begin{subequations}
	\label{ti-ode}
	\begin{align}
		\frac{{dz_i^t}}{{dt}} = {f^a}(\sum\nolimits_{j \in V} {\sum\limits_{k = 1}^K {w_{i,j,k}^t} f_k^r(z_i^t,z_j^t)} ) - z_i^t, \label{ti-odea}\\
		\overline w_{i,k}^t = \sigma (FNN_k^a(z_i^t + T{E_{(t)}})), \label{ti-odeb}\\
		\hat w_{j,k}^t = \sigma (FNN_k^r(z_j^t + T{E_{(t)}})), \label{ti-odec}\\
		w_{i,j,k}^t = \hat w_{j,k}^t \cdot \overline w_{i,k}^t, \label{ti-oded}
	\end{align}
\end{subequations}
where $f_k^r$ denotes the $k$-th basis function parameterized by an FNN, and $K$ is the total number of basis functions. The coefficient $w_{i,j,k}^t$ represents the weight of the $k$-th basis function between nodes $i$ and $j$ at time $t$. We further decompose this coefficient into two time-dependent learnable factors: $\overline w_{i,k}^t$, which denotes the response weight (receptivity) of node $i$ to the $k$-th basis function, and $\hat w_{j,k}^t$, which denotes the activation weight (propagation strength) of node $j$ for the same basis function. The joint effect of response and activation weights enables the model to capture role-specific differences of nodes under different interaction types. We define ${f^a}$ as an aggregation function parameterized by an FNN, which integrates interaction information from neighboring nodes and determines how a node should leverage these messages to update its own state. In this way, the model flexibly captures the influence of neighborhood information on node evolution. By solving the ODE, we can obtain the latent states at the target time steps:
\begin{equation}
	\label{odesolve}
	\begin{array}{l}
		z_i^{{T_h}+1},...,z_i^{T} = ODESolve(\textcolor{blue}{\cref{ti-ode}},{Z^{0}},({T_h}+1,{T_h}+2,...,T)),
	\end{array}
\end{equation}
where ${Z^{0}} = [z_1^{0},z_2^{0},...,z_N^{0}]$. Since neural ODEs generally do not admit closed-form analytical solutions, the ODESolve approximates the solution using numerical methods \cite{butcher2016numerical}, such as the Euler method or Runge–Kutta schemes, thereby generating the latent states at different time steps. In the following, we present theoretical analyses of TI-ODE.
\begin{theorem}
	\label{Existence unique}
	Assume that all FNNs in the \textcolor{blue}{\cref{ti-ode}} have bounded weights (with finite spectral norms) and bounded biases, the activation functions are Lipschitz continuous, and the numbers of nodes and basis functions are finite. Then, for any bounded initial value $z(0)\in {\mathbb{R}^{N \times d}}$(where $d$ is a fixed finite dimension), the \textcolor{blue}{\cref{ti-ode}} admit solutions that depend continuously on the initial conditions.
\end{theorem}

We next present a theorem showing that, under certain conditions, TI-ODE exhibits superior stability in the Lyapunov sense compared to a unified interaction model $\frac{{dz_i^t}}{{dt}} = {f^a}(\sum\limits_{j \in V} {{f^r}(z_i^t,z_j^t)} )$ in which all nodes share the same time-invariant interaction mechanism, when the initial states are perturbed. Specifically, during long-term evolution, TI-ODE demonstrates a lower growth rate of perturbation-induced errors and smaller accumulated errors, thereby exhibiting stronger robustness.

\begin{theorem}
	\label{Robust}
	Assume that the \textcolor{blue}{\cref{ti-ode}} admit solutions that depend continuously on the initial conditions. Let the Lipschitz constants of ${f^a}$,  ${f^r}$,$f_k^r$, and $w_{i,j,k}^t$ be ${L_a}$, ${L_r}$ , ${L_{r,k}}$ and ${L_{w,k}}$, respectively. Also, let the upper bounds of $f_k^r(z_i^t,z_j^t)$ and $w_{i,j,k}^t$ be ${C_{r,k}}$ and ${C_{w,k}}$. Define ${L^k} = d{C_{r,k}}{L_{w,k}} + \sqrt d {C_{w,k}}{L_{r,k}}$. If the condition $\sum\limits_k {{L^k}}  \le {L_r}$ holds, then the \textcolor{blue}{\cref{ti-ode}} exhibit stronger robustness under the Lyapunov function $V(t) = \frac{1}{2}||{e^t}|{|^2}$ than the unified-interaction model.
\end{theorem}

Because our architecture explicitly decomposes the model’s overall representational capacity into multiple basis functions and their weighted combination, this modular design can effectively reduce the range of gradient variation within any single subnetwork. In addition, the combination weights are modulated by nonlinear activation functions, which can further suppress potential nonlinear amplification effects. As a result, the overall gradient norm can be reasonably distributed across the $K$ basis functions, and under the joint influence of the weight distribution and the activations, the Lipschitz constant of each individual basis function is typically substantially smaller than the Lipschitz constant of the interaction function in a unified-interaction model. We therefore regard this assumption as reasonable: under mild conditions, this design tends to yield a tighter upper bound on the Lipschitz constant of the overall model, thereby improving robustness to perturbations. Detailed proofs of the above theorems are provided in the \textcolor{blue}{\cref{theorem1,theorem2}}.

\subsection{Decoder and Optimization}
\begin{algorithm}[!h]
	\caption{Training Algorithm of TI-ODE}
	
	\begin{algorithmic}[1]
		\STATE \textbf{Input:} Observation data $\mathcal{G}^{1:T_{h}}$.     
		\STATE \textbf{Output:} The parameters in the model.
		\STATE Initialize model parameters;
		\STATE Construct the temporal graph ${\cal A}$ as defined in \textcolor{blue}{\cref{st-graph}};
		\WHILE{not convergence}
		\FOR{each training sequence}
		
		\STATE Encode the temporal graph ${\cal A}$ to obtain node sequence representations $u_i$ as defined in \textcolor{blue}{\cref{attention,update,u_attention1,u_attention2}};
		\STATE Generate the primary posterior distribution $(\mu_i,\sigma_i)$ from $u_i$ according to \textcolor{blue}{\cref{original posterior}};
		\STATE Compute multiple auxiliary posterior distributions $(\mu _{i,k}^r,\sigma _{i,k}^r)$ from $u_i$ according to \textcolor{blue}{\cref{aul posterior1,aul posterior2}};
		\STATE Fuse the primary and auxiliary posteriors to obtain the refined posterior distribution ${\cal N}({\hat \mu _i},{\hat \sigma _i})$ according to \textcolor{blue}{\cref{aul posterior3,aul posterior4}};
		\STATE Sample the initial latent state $z_i^0 \sim {\cal N}({\hat \mu _i},{\hat \sigma _i})$;	
		\STATE Solve our TI-ODE in \textcolor{blue}{\cref{odesolve}};
		\STATE Decode latent states at future times to output predicted states;
		\STATE Compute the ELBO objective $\mathcal{L}$, i.e., \textcolor{blue}{\cref{loss}};
		\STATE Do backward propagation with $\mathcal{L}$;
		\ENDFOR
		\ENDWHILE
		
	\end{algorithmic}
	\label{Training Algorithm of TI-ODE}
\end{algorithm}
To map the latent states back to node trajectories in the observation space, we design a lightweight decoder for each node. Given the latent state $z_i^t$ at any time step, the decoder outputs the predicted future positions in the observation space, i.e., $\hat x_i^t = {\varphi ^d}(z_i^t)$, where ${\varphi ^d}$ is parameterized by an FNN. The entire framework is trained under variational inference. We sample the initial latent state from the refined posterior distribution inferred from the historical dynamic graph sequence, evolve it with TI-ODE to obtain future latent states, and then use a decoder to produce predictions in the observation space. The training objective is the evidence lower bound (ELBO) of the log-likelihood:

\begin{multline}
	\label{loss}
	\mathcal{L} 
	= {\mathbb{E}_{{Z^{0}}\sim { \prod\limits_{i = 1}^N {q(z_i^0|{{\cal G}^{1:{T_{h}}}})} }}}[\log p({X^{{T_{h}} + 1:T}})]
	\quad- KL[\prod\limits_{i = 1}^N {q(z_i^{0}|{{\cal G}^{1:{T_{h}}}})||{\cal N}(0,I)} ]
	\\
	= - \sum\limits_{i = 1}^N {\sum\limits_{t = {T_{h}} + 1}^{{ T}} {\frac{{||x_i^t - \hat x_i^t||}}{{2\sigma _c^2}}} }
	\quad- KL[\prod\limits_{i = 1}^N {q(z_i^{0}|{{\cal G}^{1:{T_{h}}}})||{\cal N}(0,I)} ],
\end{multline}
where ${\cal N}(0,I)$ is a standard normal prior distribution, and $\sigma _c^2$ denotes the fixed observation noise variance. The training process of this model is detailed in \textcolor{blue}{\cref{Training Algorithm of TI-ODE}}.

\section{Experiments}
\subsection{Datasets}
We conducted experiments on six datasets spanning three categories: physical dynamics simulations, molecular dynamics simulations, and real-world data. The statistical summary of these datasets is shown in \textcolor{blue}{\cref{Statistics}}.
\begin{table}[!h]
	\centering
	\caption{Statistics of the datasets used in the experiments.}
	\label{Statistics}
	\begin{tabular}{ccccccc}
		\hline
		\toprule[1pt]
		Dataset                      & \textit{Spring}     & \textit{Charged}     & \textit{2N5C}       & \textit{5AWL}       & \textit{Motion} & \textit{Covid} \\ \hline
		\#Nodes                      & \multicolumn{2}{c}{10}   & 60         & 40         & 29     & 31    \\ \cmidrule(r){2-3}   \cmidrule(r){4-5} \cmidrule(r){6-6} \cmidrule(r){7-7}       
		\#Nodes ft                   & \multicolumn{2}{c}{4}    & \multicolumn{2}{c}{3}   & 6      & 2     \\
		\#Train                      & \multicolumn{2}{c}{1000} & \multicolumn{2}{c}{200} & 222    & 21    \\
		\#Valid                      & \multicolumn{2}{c}{200}  & \multicolumn{2}{c}{50}  & 74     & -     \\
		\#Test                       & \multicolumn{2}{c}{200}  & \multicolumn{2}{c}{50}  & 74     & 8     \\
		\toprule[1pt]
	\end{tabular}
	\vspace{-0.3cm}
\end{table}
\subsubsection{Physical Dynamics Simulations}
The category of physical dynamics simulations comprises the \textit{Spring} \citep{kipf2018neural} and \textit{Charged} \citep{kipf2018neural} datasets. The \textit{Spring} dataset is based on a classical physical simulation where each node represents a point mass and each edge denotes a spring connecting two masses. Governed by Hooke’s law, the nodes move continuously under the influence of elastic forces and damping. The \textit{Charged} dataset simulates the motion of particles subject to Coulomb interactions. Here, nodes represent charged particles, and the edges reflect the repulsive or attractive forces determined by their charges. 
\subsubsection{Molecular Dynamics Simulations}
The category of molecular dynamics simulations comprises the \textit{2N5C} \citep{luopgode} and \textit{5AWL} \citep{luopgode} datasets. The \textit{2N5C} and \textit{5AWL} datasets consist of molecular dynamics data derived from the Protein Data Bank (PDB\footnote{https://www.rcsb.org}). In these datasets, each node represents an atom or residue in the protein structure, while edges denote physical spatial adjacency in three-dimensional space. The node features include time-varying spatial coordinates obtained from molecular dynamics simulations or experimental trajectories. Specifically, \textit{2N5C} is a protein complex with high structural flexibility. Its trajectories show pronounced local motions as well as global conformational changes. In contrast, \textit{5AWL} has a more stable tertiary structure and exhibits smoother trajectory variations.
\subsubsection{Real-world Datasets}
The category of real-world datasets comprises the \textit{Motion} \citep{cmu_mocap} and \textit{Covid} datasets. The \textit{Motion} dataset consists of three-dimensional position and angle sequences of human joints. In this graph, nodes correspond to key joints of the human body, and edges represent skeletal connectivity or structural constraints. Collected using motion capture systems, the data include natural movements such as walking, running, and jumping. The data used in our experiments are sourced from the CMU Motion Capture Dataset (CMU MoCap\footnote{http://mocap.cs.cmu.edu}), a widely used large-scale motion sequence repository. 

We constructed a \textit{Covid} dataset based on publicly available COVID-19 epidemic reports in China. The \textit{Covid} dataset was compiled from daily epidemic reports publicly released by provincial health commissions and the National Health Commission of China (NHC\footnote{https://www.nhc.gov.cn/}), covering all provincial-level COVID-19 cases in mainland China from February 2020 to December 2022. In the constructed graph, nodes correspond to the 31 provincial-level administrative regions (excluding Hong Kong, Macau, and Taiwan), and node features consist of the cumulative confirmed cases and daily new cases over time. Given the widespread inter-provincial connections via road, rail, and air transportation, the graph is modeled as fully connected, allowing for potential transmission pathways between any two provinces. This dataset spans multiple viral phases, including the original strain, Delta, and Omicron \citep{niu2022comparison}, and encompasses various control policies, such as lockdowns and dynamic COVID-zero strategy \citep{liu2022dynamic}, as well as differing healthcare resource pressures \citep{wang2023quantitative} and population mobility patterns \citep{li2021changes}. The evolution of these factors across different time stages alters the underlying mechanism by which node states are generated, thereby inducing temporal drift in the distribution of node features. Therefore, this dataset not only provides a challenging testbed for evaluating the generalization ability of TI-ODE in dynamic graph settings, but also offers an experimental scenario for analyzing the time-varying interaction mechanisms learned by TI-ODE. To preserve training capacity and temporal feature learning on this small-scale dataset, we opted for a simple train-test split instead of a three-way partition. Since hyperparameters were adopted from prior experiments without further tuning, this approach maximizes data utility and effectively assesses the model's generalization across out-of-distribution temporal windows.

\begin{table}[h]
	\centering
	\caption{Hyperparameter setup for every dataset.}
	\label{Hyperparameter setup}
		\begin{tabular}{ccccccc}
			\hline
			\toprule[1pt]
			Hyper             & \multicolumn{6}{c}{Values}                      \\ \cline{2-7} 
			parameter         & \textit{Spring} & \textit{Charged} & \textit{2N5C} & \textit{5AWL} & \textit{Motion} & \textit{Covid} \\ \hline
			Batch             & 64     & 64      & 2    & 4    & 8      & 16    \\\cline{2-7} 
			Learn rate        & \multicolumn{6}{c}{0.0001}                      \\
			Dropout           & \multicolumn{6}{c}{0.2}                         \\
			Optimizer         & \multicolumn{6}{c}{Adam}                        \\
			ODESolve          & \multicolumn{6}{c}{RK4}                         \\
			Bias layers       & \multicolumn{6}{c}{1}                           \\
			Bias hidden dim   & \multicolumn{6}{c}{128}                         \\
			Bias nums         & \multicolumn{6}{c}{5}                           \\
			Random layers     & \multicolumn{6}{c}{1}                           \\
			Random hidden dim & \multicolumn{6}{c}{64}                          \\
			Random nums       & \multicolumn{6}{c}{4}                           \\ \hline
			\toprule[1pt]
		\end{tabular}
	\vspace{-0.5cm}
\end{table}
\subsection{Baselines}
To evaluate the performance of TI-ODE, we consider 11 baseline models: the time-series model Latent-ODE \citep{rubanova2019latent} and the edge-feature-explicit Edge-GNN \citep{gong2019exploiting}; classic dynamic GNNs (DCRNN \citep{li2018diffusion}, GraphWaveNet \citep{wu2019graph}, AGCRN \citep{bai2020adaptive},  TTS-AMP \citep{cini2023taming}, and T\&S-AMP \citep{cini2023taming}); and graph neural ODE frameworks (LG-ODE \citep{huang2020learning}, CG-ODE \citep{huang2021coupled}, PG-ODE \citep{luopgode}, and CSG-ODE \citep{wangcsg}).

\begin{table*}[h]
	\centering
	\caption{The MSE and MAE ($ \times {10^{ - 2}}$) of compared methods on physical dynamics simulations.}
	\label{compared physical}
	\begin{adjustbox}{width=\textwidth}
		\begin{tabular}{ccccccccc}
			\hline
			\toprule[1pt]
			\textbf{Dataset}          & \multicolumn{4}{c}{\textit{\textbf{Spring}}} &\multicolumn{4}{c}{\textit{\textbf{Charged}}}  \\ \cmidrule(r){1-1} \cmidrule(r){2-5}   \cmidrule(r){6-9} 
			Pred Len         & \multicolumn{2}{c}{12}& \multicolumn{2}{c}{24} & \multicolumn{2}{c}{12} & \multicolumn{2}{c}{24}    \\ \cmidrule(r){1-1}  \cmidrule(r){2-3} \cmidrule(r){4-5}   \cmidrule(r){6-7} \cmidrule(r){8-9}  
			Metric            & MSE           & MAE           & MSE           & MAE           & MSE           & MAE           & MSE           & MAE \\ \hline
			Latent-ODE     &1.32 $\pm$ 0.10 &8.54 $\pm$ 0.27 &3.48 $\pm$ 0.07 &13.74 $\pm$ 0.16 &3.19 $\pm$ 0.22 &13.57 $\pm$ 0.48 &5.76 $\pm$ 0.31 &17.41 $\pm$ 1.28 \\
			Edge-GNN     &1.06 $\pm$ 0.08 &7.61 $\pm$ 0.34 &2.42 $\pm$ 0.15 &11.97 $\pm$ 0.56 &2.27 $\pm$ 0.09 &11.12 $\pm$ 0.23 &3.62 $\pm$ 0.10 &14.38 $\pm$ 0.73 \\ \hline
			
			DCRNN     &0.62 $\pm$ 0.04 &5.35 $\pm$ 0.20 &1.02 $\pm$ 0.11 &0.62 $\pm$ 0.30 &1.78 $\pm$ 0.10 &10.37 $\pm$ 0.56 &2.84 $\pm$ 0.11 &12.39 $\pm$ 0.32 \\ 
			GraphWaveNet     &0.50 $\pm$ 0.03 &5.05 $\pm$ 0.17 &0.63 $\pm$ 0.04 &5.64 $\pm$ 0.18 &1.55 $\pm$ 0.03 &9.91 $\pm$ 0.19 &2.63 $\pm$ 0.07 &12.26 $\pm$ 0.25 \\
			AGCRN    &0.48 $\pm$ 0.02 &4.65 $\pm$ 0.32 &0.74 $\pm$ 0.05 &5.74 $\pm$ 0.15 &1.58 $\pm$ 0.10 &10.13 $\pm$ 0.32 &2.67 $\pm$ 0.02 &11.93 $\pm$ 0.15 \\
			TTS-AMP    &0.41 $\pm$ 0.03 &4.54 $\pm$ 0.12 &0.63 $\pm$ 0.07 &5.67 $\pm$ 0.34 &1.64 $\pm$ 0.07 &9.98 $\pm$ 0.47 &2.68 $\pm$ 0.08 &11.76 $\pm$ 0.21 \\
			T\&S-AMP  &0.46 $\pm$ 0.03 &4.83 $\pm$ 0.15 &0.72 $\pm$ 0.04 &6.01 $\pm$ 0.21 &1.62 $\pm$ 0.04 &9.59 $\pm$ 0.30 &2.66 $\pm$ 0.07 &11.99 $\pm$ 0.17 \\ \hline
			LG-ODE 	  &0.58 $\pm$ 0.02 &5.84 $\pm$ 0.09 &1.02 $\pm$ 0.08 &7.80 $\pm$ 0.19 &1.65 $\pm$ 0.04 &10.12 $\pm$ 0.47 &2.71 $\pm$ 0.06 &11.77 $\pm$ 0.10 \\
			CG-ODE	  &0.45 $\pm$ 0.02 &5.39 $\pm$ 0.10 &0.70 $\pm$ 0.02 &6.55 $\pm$ 0.06 &1.68 $\pm$ 0.13 &10.18 $\pm$ 0.25 &2.64 $\pm$ 0.07 &11.97 $\pm$ 0.16 \\
			PG-ODE	  &0.32 $\pm$ 0.01 &4.40 $\pm$ 0.04 &0.51 $\pm$ 0.03 &5.64 $\pm$ 0.15 &1.53 $\pm$ 0.08 &9.58 $\pm$ 0.06 &1.91 $\pm$ 0.07 &11.45 $\pm$ 0.46 \\
			CSG-ODE	  &0.23 $\pm$ 0.02 &3.79 $\pm$ 0.16 &0.45 $\pm$ 0.03 &5.41 $\pm$ 0.12 &1.47 $\pm$ 0.01 &9.68 $\pm$ 0.23 &1.81 $\pm$ 0.06 &10.38 $\pm$ 0.28 \\
			\rowcolor{red!10} 	TI-ODE	  &\textbf{0.16 $\pm$ 0.01} &\textbf{3.16 $\pm$ 0.03} &\textbf{0.30 $\pm$ 0.01} &\textbf{4.36 $\pm$ 0.05} &\textbf{1.31 $\pm$ 0.03} &\textbf{8.87 $\pm$ 0.39} &\textbf{1.65 $\pm$ 0.04} &\textbf{10.06 $\pm$ 0.22 }
			\\ \hline \toprule[1pt]  
			
		\end{tabular}
	\end{adjustbox}
\end{table*}
\begin{table*}[h]
	\centering
	\caption{The MSE and MAE ($ \times {10^{ - 2}}$) of compared methods on molecular dynamics simulations.}
	\label{compared molecular}
	\begin{adjustbox}{width=\textwidth}
		\begin{tabular}{ccccccccc}
			\hline
			\toprule[1pt]
			\textbf{Dataset}          & \multicolumn{4}{c}{\textit{\textbf{2N5C}}} &\multicolumn{4}{c}{\textit{\textbf{5AWL}}}  \\ \cmidrule(r){1-1} \cmidrule(r){2-5}   \cmidrule(r){6-9} 
			Pred Len         & \multicolumn{2}{c}{12}& \multicolumn{2}{c}{24} & \multicolumn{2}{c}{12} & \multicolumn{2}{c}{24}    \\ \cmidrule(r){1-1}  \cmidrule(r){2-3} \cmidrule(r){4-5}   \cmidrule(r){6-7} \cmidrule(r){8-9}  
			Metric            & MSE           & MAE           & MSE           & MAE           & MSE           & MAE           & MSE           & MAE \\ \hline
			Latent-ODE     &1.68 $\pm$ 0.01 &10.19 $\pm$ 0.23 &2.31 $\pm$ 0.16 &12.44 $\pm$ 0.47 &1.26 $\pm$ 0.25 &8.79 $\pm$ 0.59 &1.71 $\pm$ 0.06 &10.25 $\pm$ 0.43 \\
			Edge-GNN     &1.15 $\pm$ 0.08 &8.72 $\pm$ 0.38 &1.66 $\pm$ 0.11 &10.33 $\pm$ 0.65 &0.79 $\pm$ 0.04 &7.10 $\pm$ 0.29 &1.16 $\pm$ 0.14 &8.29 $\pm$ 0.31 \\ \hline
			
			DCRNN     &0.81 $\pm$ 0.02 &6.87 $\pm$ 0.06 &0.97 $\pm$ 0.02 &7.58 $\pm$ 0.09 &0.72 $\pm$ 0.03 &6.49 $\pm$ 0.04 &0.94 $\pm$ 0.01 &7.30 $\pm$ 0.04  \\ 
			GraphWaveNet     &0.53 $\pm$ 0.04 &5.15 $\pm$ 0.17 &0.83 $\pm$ 0.08 &6.17 $\pm$ 0.31 &0.38 $\pm$ 0.01 &4.49 $\pm$ 0.11 &0.42 $\pm$ 0.02 &4.75 $\pm$ 0.13 \\
			AGCRN    &0.60 $\pm$ 0.03 &5.61 $\pm$ 0.14 &0.88 $\pm$ 0.08 &6.75 $\pm$ 0.38 &0.36 $\pm$ 0.02 &4.30 $\pm$ 0.05 &0.46 $\pm$ 0.02 &0.50 $\pm$ 0.09 \\
			TTS-AMP    &0.53 $\pm$ 0.06 &5.26 $\pm$ 0.32 &0.82 $\pm$ 0.04 &6.65 $\pm$ 0.19  &0.30 $\pm$ 0.03 &4.09 $\pm$ 0.16 &0.38 $\pm$ 0.01 &4.53 $\pm$ 0.06\\
			T\&S-AMP  &0.58 $\pm$ 0.06 &5.41 $\pm$ 0.30 &0.87 $\pm$ 0.01 &6.67 $\pm$ 0.14 &0.34 $\pm$ 0.01 &4.17 $\pm$ 0.05 &0.44 $\pm$ 0.04 &4.82 $\pm$ 0.24 \\ \hline
			LG-ODE 	  &0.57 $\pm$ 0.02 &6.03 $\pm$ 0.13 &0.79 $\pm$ 0.10 &6.80 $\pm$ 0.51 &0.36 $\pm$ 0.01 &4.73 $\pm$ 0.10 &0.47 $\pm$ 0.02 &5.46 $\pm$ 0.09 \\
			CG-ODE	  &0.56 $\pm$ 0.07 &5.94 $\pm$ 0.37 &0.84 $\pm$ 0.01 &6.96 $\pm$ 0.05 &0.31 $\pm$ 0.02 &4.42 $\pm$ 0.13 &0.33 $\pm$ 0.01 &4.63 $\pm$ 0.10 \\
			PG-ODE	  &0.42 $\pm$ 0.01 &4.99 $\pm$ 0.04 & 0.65 $\pm$ 0.02&6.34 $\pm$ 0.14 &0.27 $\pm$ 0.03 &4.08 $\pm$ 0.10 &0.30 $\pm$ 0.01 &4.59 $\pm$ 0.07 \\
			CSG-ODE	  &0.37 $\pm$ 0.02 &4.66 $\pm$ 0.15 &0.63 $\pm$ 0.01 &6.26 $\pm$ 0.09 &0.19 $\pm$ 0.02 &3.42 $\pm$ 0.19 &0.25 $\pm$ 0.02 &3.97 $\pm$ 0.14 \\
			\rowcolor{red!10} 	TI-ODE	  &\textbf{0.31 $\pm$ 0.01} &\textbf{4.29 $\pm$ 0.03} &\textbf{0.49 $\pm$ 0.02} &\textbf{5.56 $\pm$ 0.18} &\textbf{0.14 $\pm$ 0.02} &\textbf{3.07 $\pm$ 0.03} &\textbf{0.18 $\pm$ 0.01 }&\textbf{3.43 $\pm$ 0.08 }
			\\ \hline \toprule[1pt]  
			
		\end{tabular}
	\end{adjustbox}
\end{table*}
\begin{table*}[h]
	\centering
	\caption{The MSE and MAE ($ \times {10^{ - 2}}$) of compared methods on real-world datasets.}
	\label{compared real-world}
	\begin{adjustbox}{width=\textwidth}
		\begin{tabular}{ccccccccc}
			\hline
			\toprule[1pt]
			\textbf{Dataset}          & \multicolumn{4}{c}{\textit{\textbf{Motion}}} &\multicolumn{4}{c}{\textit{\textbf{Covid}}}  \\ \cmidrule(r){1-1} \cmidrule(r){2-5}   \cmidrule(r){6-9} 
			Pred Len         & \multicolumn{2}{c}{12}& \multicolumn{2}{c}{24} & \multicolumn{2}{c}{12} & \multicolumn{2}{c}{24}    \\ \cmidrule(r){1-1}  \cmidrule(r){2-3} \cmidrule(r){4-5}   \cmidrule(r){6-7} \cmidrule(r){8-9}  
			Metric            & MSE           & MAE           & MSE           & MAE           & MSE           & MAE           & MSE           & MAE \\ \hline
			Latent-ODE     &4.94 $\pm$ 0.26 &15.46 $\pm$ 0.63 &6.69 $\pm$ 0.20 &17.71 $\pm$ 0.59 &6.49 $\pm$ 0.10 &17.62 $\pm$ 0.87 &7.13 $\pm$ 0.16 &18.19 $\pm$ 0.23 \\
			Edge-GNN     &4.11 $\pm$ 0.52 &14.38 $\pm$ 1.33 &5.60 $\pm$ 0.47 &16.47 $\pm$ 1.03 &4.34 $\pm$ 0.05 &15.58 $\pm$ 0.21 &5.27 $\pm$ 0.03 &16.64 $\pm$ 0.94 \\ \hline
			
			DCRNN     &0.87 $\pm$ 0.05 &5.13 $\pm$ 0.04 &1.06 $\pm$ 0.06 &5.29 $\pm$ 0.07 &6.49 $\pm$ 0.10 &14.61 $\pm$ 0.32 &7.13 $\pm$ 0.16 &15.19 $\pm$ 0.23 \\ 
			GraphWaveNet     &0.82 $\pm$ 0.05 &4.38 $\pm$ 0.11 &0.94 $\pm$ 0.06 &4.68 $\pm$ 0.16  &3.48 $\pm$ 0.26 &13.48 $\pm$ 0.80 &3.98 $\pm$ 0.19 &13.79 $\pm$ 1.03 \\
			AGCRN    &0.75 $\pm$ 0.05 &5.38 $\pm$ 0.25 &0.92 $\pm$ 0.07 &6.00 $\pm$ 0.21 &2.35 $\pm$ 0.19 &9.37 $\pm$ 0.61 &2.81 $\pm$ 0.26 &10.49 $\pm$ 0.29 \\
			TTS-AMP   &0.67 $\pm$ 0.06 &4.07 $\pm$ 0.17 &0.82 $\pm$ 0.05 &5.02 $\pm$ 0.14 &2.33 $\pm$ 0.03 &9.32 $\pm$ 0.43 &2.39 $\pm$ 0.04 &9.80 $\pm$ 0.17 \\
			T\&S-AMP  &0.71 $\pm$ 0.02 &4.39 $\pm$ 0.18 &0.87 $\pm$ 0.01 &5.01 $\pm$ 0.43 &2.37 $\pm$0.06 &9.54 $\pm$0.32 &2.62 $\pm$0.11 &9.70 $\pm$0.29 \\ \hline
			LG-ODE 	  & 0.80 $\pm$ 0.03&4.64 $\pm$ 0.16 &0.84 $\pm$ 0.04 &4.64 $\pm$ 0.19 &3.49 $\pm$ 0.14 &13.69 $\pm$ 0.67 &4.13 $\pm$ 0.05 &13.88 $\pm$ 0.20 \\
			CG-ODE	  &0.95 $\pm$ 0.01 &4.97 $\pm$ 0.05 &0.99 $\pm$ 0.16 &5.07 $\pm$ 0.02 &3.06 $\pm$ 0.09 &12.23 $\pm$ 0.94 &3.56 $\pm$ 0.10 &13.25 $\pm$ 1.19 \\
			PG-ODE	  &0.54 $\pm$ 0.04 &4.41 $\pm$ 0.17 &0.64 $\pm$ 0.03 &4.52 $\pm$ 0.33 &2.32 $\pm$ 0.04 &8.75 $\pm$ 0.18 &2.35 $\pm$ 0.05 &9.58 $\pm$ 0.31 \\
			CSG-ODE	  &0.47 $\pm$ 0.02 &4.00 $\pm$ 0.13 &0.50 $\pm$ 0.08 &4.34 $\pm$ 0.24 &2.18 $\pm$ 0.01 &9.42 $\pm$ 1.54 &2.31 $\pm$ 0.07 &10.29 $\pm$ 1.03 \\
			\rowcolor{red!10} 	TI-ODE	  & \textbf{0.16 $\pm$ 0.01} &\textbf{2.09 $\pm$ 0.05} &\textbf{0.20 $\pm$ 0.01} &\textbf{2.42 $\pm$ 0.04} &\textbf{2.06 $\pm$ 0.02} &\textbf{8.19 $\pm$ 0.09 }&\textbf{2.20 $\pm$ 0.02} &\textbf{9.36 $\pm$ 0.05 }
			\\ \hline \toprule[1pt]  
			
		\end{tabular}
	\end{adjustbox}
\end{table*}
\begin{table*}[ht]
	\centering
	\caption{Ablation study on all datasets ($ \times {10^{ - 2}}$).}
	\label{Ablation}
	\begin{adjustbox}{width=\textwidth}
		\begin{tabular}{ccccccccccccccccccccccccc}
			\hline
			\toprule[1pt]
			Dataset         & \multicolumn{4}{c}{\textit{Spring}}                           & \multicolumn{4}{c}{\textit{Charged}}                           & \multicolumn{4}{c}{\textit{2N5C}}                             & \multicolumn{4}{c}{\textit{5AWL}}                             & \multicolumn{4}{c}{\textit{Motion}}   &\multicolumn{4}{c}{\textit{Covid}}                          \\ \cmidrule(r){1-1} \cmidrule(r){2-5}   \cmidrule(r){6-9} \cmidrule(r){10-13} \cmidrule(r){14-17}  \cmidrule(r){18-21} \cmidrule(r){22-25}
			Pred Len        & \multicolumn{2}{c}{12}        & \multicolumn{2}{c}{24}        & \multicolumn{2}{c}{12}        & \multicolumn{2}{c}{24}         & \multicolumn{2}{c}{12}        & \multicolumn{2}{c}{24}        & \multicolumn{2}{c}{12}        & \multicolumn{2}{c}{24}        & \multicolumn{2}{c}{12}        & \multicolumn{2}{c}{24}       & \multicolumn{2}{c}{12}        & \multicolumn{2}{c}{24} \\ \cmidrule(r){1-1} \cmidrule(r){2-3} \cmidrule(r){4-5}   \cmidrule(r){6-7} \cmidrule(r){8-9}  \cmidrule(r){10-11}\cmidrule(r){12-13}  \cmidrule(r){14-15}\cmidrule(r){16-17}   \cmidrule(r){18-19} \cmidrule(r){20-21}  \cmidrule(r){22-23} \cmidrule(r){24-25}
			Metric          & MSE           & MAE           & MSE           & MAE           & MSE           & MAE           & MSE           & MAE            & MSE           & MAE           & MSE           & MAE           & MSE           & MAE           & MSE           & MAE           & MSE           & MAE           & MSE           & MAE     & MSE           & MAE           & MSE           & MAE       \\ \hline
			TI-ODE w/o W    & 0.29          & 4.23          & 0.40          & 4.99          & 1.60          & 9.69          & 1.81          & 10.36          & 0.30          & 4.37          & 0.52          & 5.66          & 0.16          & 3.20          & 0.21          & 3.61          & 0.25          & 2.69          & 0.47          & 3.80        &2.33 &9.59 &2.37 &9.74    \\
			TI-ODE w/o R    & 0.20          & 3.44          & 0.42          & 5.01          & 1.78          & 10.08         & 1.83          & 10.52          & 0.31          & 4.47          & 0.60          & 5.86          & 0.15          & 3.09          & 0.21          & 3.63          & 0.16          & 2.21          & 0.22          & 2.57     &2.18 &9.38 &2.19 & 9.35     \\
			TI-ODE w/o O    & 0.25          & 3.84          & 0.45          & 5.14          & 1.46          & 9.95          & 1.80          & 10.47          & 0.31          & 4.34          & 0.50          & 5.57          & 0.15          & 3.10          & 0.19          & 3.48          & 0.18          & 2.40          & 0.23          & 2.73        &2.36 &9.86 &2.46 &10.32    \\
			\rowcolor{red!10}TI-ODE & \textbf{0.16} & \textbf{3.01} & \textbf{0.31} & \textbf{4.28} & \textbf{1.31} & \textbf{9.39} & \textbf{1.65} & \textbf{10.01} & \textbf{0.29} & \textbf{4.25} & \textbf{0.48} & \textbf{5.35} & \textbf{0.14} & \textbf{3.05} & \textbf{0.18} & \textbf{3.37} & \textbf{0.16} & \textbf{2.06} & \textbf{0.20} & \textbf{2.39} &\textbf{2.03}&\textbf{8.21}&\textbf{2.17}&\textbf{9.31}\\ \hline \toprule[1pt]
			
		\end{tabular}
	\end{adjustbox}
\end{table*}
\begin{figure*}[htbp]
	\begin{center}
		\includegraphics[width=\textwidth]{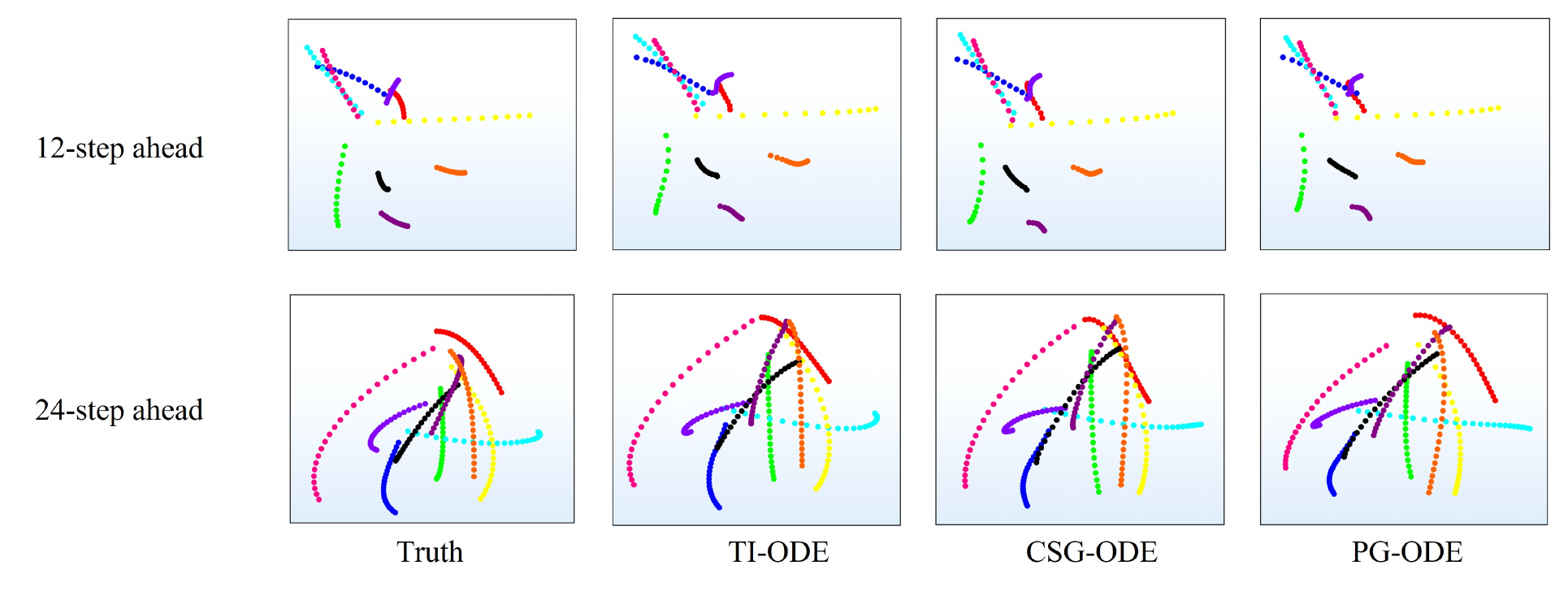}
		\caption{Visualization of the \textit{Spring} dataset under two different prediction lengths.}
		\label{visualizations main}
	\end{center}
\end{figure*}
\begin{itemize}
	\item Latent-ODE \citep{rubanova2019latent}: It introduces neural ODE in the latent space, using a continuous-time dynamical system to model the evolution of time series. Its key idea is to learn an implicit continuous dynamics field, which allows it to handle irregularly sampled inputs while naturally supporting interpolation and extrapolation at arbitrary time points.
	\item Edge-GNN \citep{gong2019exploiting}: It emphasizes the explicit use of edge features during message passing, updating node and edge representations simultaneously to enhance the model’s capacity for capturing fine-grained relational structures in graphs.
	\item DCRNN \citep{li2018diffusion}: A diffusion convolutional recurrent network that models spatiotemporal dependencies via a diffusion process on directed graphs.
	\item GraphWaveNet \citep{wu2019graph}: A spatiotemporal model that integrates dilated causal convolutions with adaptive graph convolutions to capture long-range temporal dependencies.
	\item AGCRN \citep{bai2020adaptive}: An adaptive graph convolutional recurrent network that learns hidden graph structures through node embeddings.
	\item TTS-AMP \citep{cini2023taming}: A “Temporal-Then-Spatial (TTS)” architecture that first employs a GRU to extract temporal features, and then applies an Anisotropic Message Passing (AMP) mechanism that simultaneously considers source nodes, target nodes, and edge weights, dynamically adjusting the influence of neighbors to capture complex spatial dependencies.
	\item T\&S-AMP \citep{cini2023taming}: A “Temporal-And-Spatial (T\&S)” architecture that integrates AMP operators within the GRU, enabling dynamic and collaborative modeling of spatiotemporal interactions at each time step.
	\item LG-ODE \citep{huang2020learning}: It extends neural ODE to graph settings by modeling local node dynamics in continuous time, enabling the learning of structured, time-varying patterns under irregular sampling; it is among the early works that apply the ODE paradigm to graph-structured dynamical systems. 
	\item CG-ODE \citep{huang2021coupled}: It further incorporates coupled dynamics to jointly model interactions between nodes, thereby better capturing complex dependencies in multi-body systems.
	\item PG-ODE \citep{luopgode}: It represents the latent dynamic patterns of the system using structured prototypes, decomposing system behavior into multiple prototypes and evolving them continuously in the prototype space, which allows for more flexible dynamics modeling.
	\item CSG-ODE \citep{wangcsg}: It introduces control signals to coordinate the joint evolution of node states and graph structure, while integrating stabilization designs to enhance numerical stability and robustness in highly nonlinear and complex scenarios.
\end{itemize}
\begin{figure*}[htbp]
	\begin{center}
		\includegraphics[width=\textwidth]{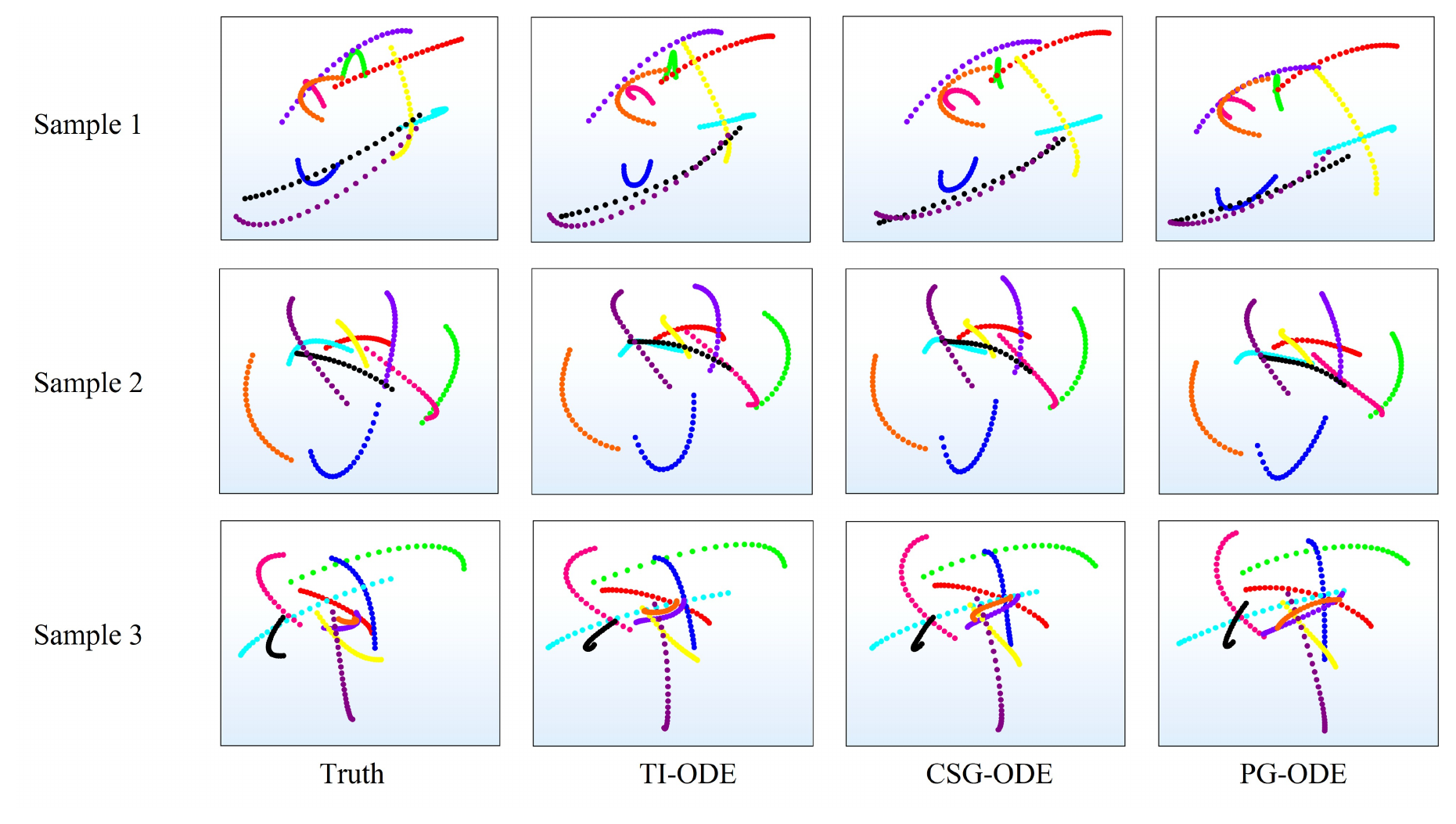}
		\caption{Visualization of the \textit{Spring} dataset (24-step ahead).}
		\label{visualizations more}
	\end{center}
\end{figure*}
\begin{figure*}[htbp]
	\centering
	\subfloat[]{\includegraphics[width=2in]{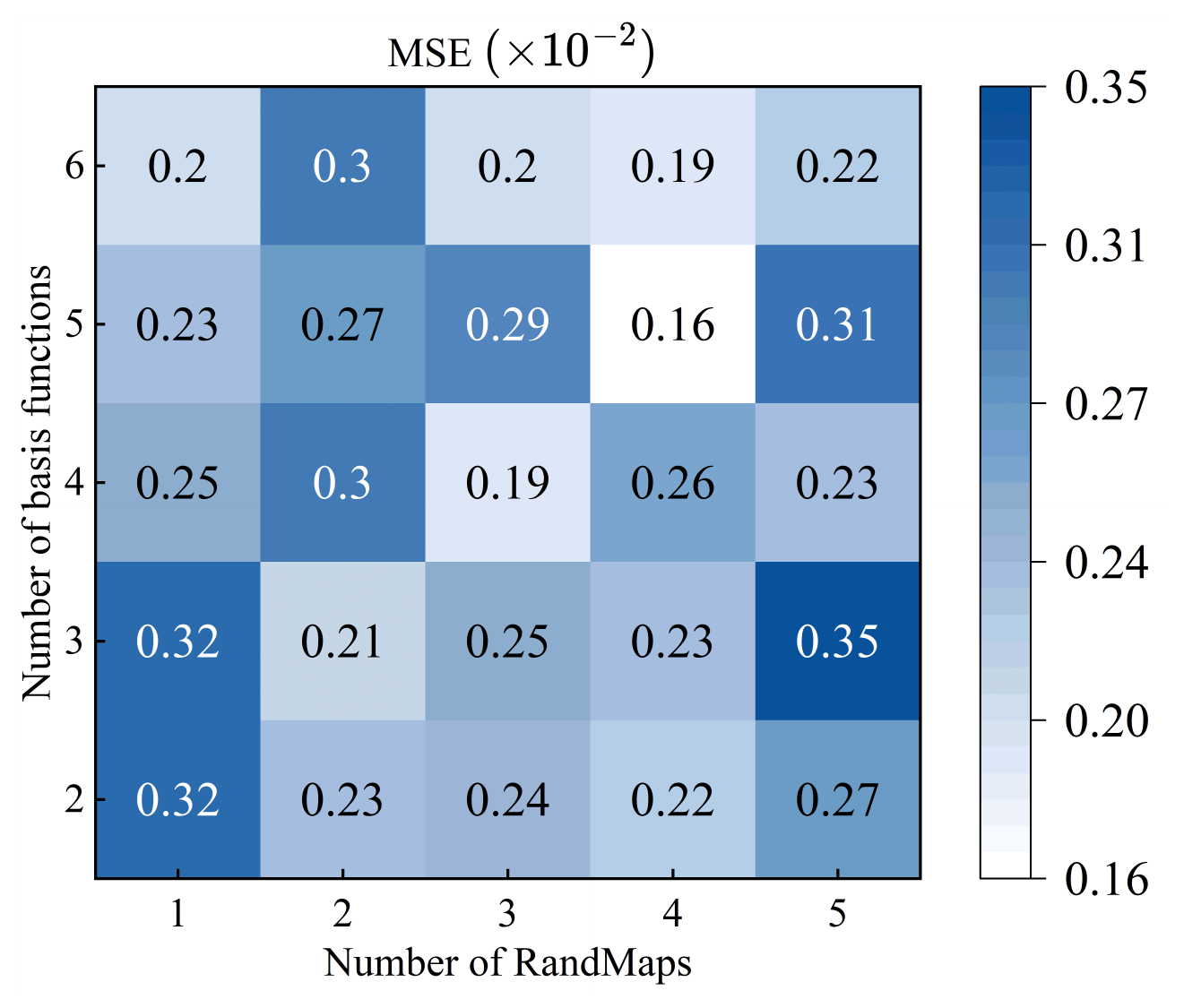}%
		\label{spring_MSE_12}}
	\hfil
	\subfloat[]{\includegraphics[width=2in]{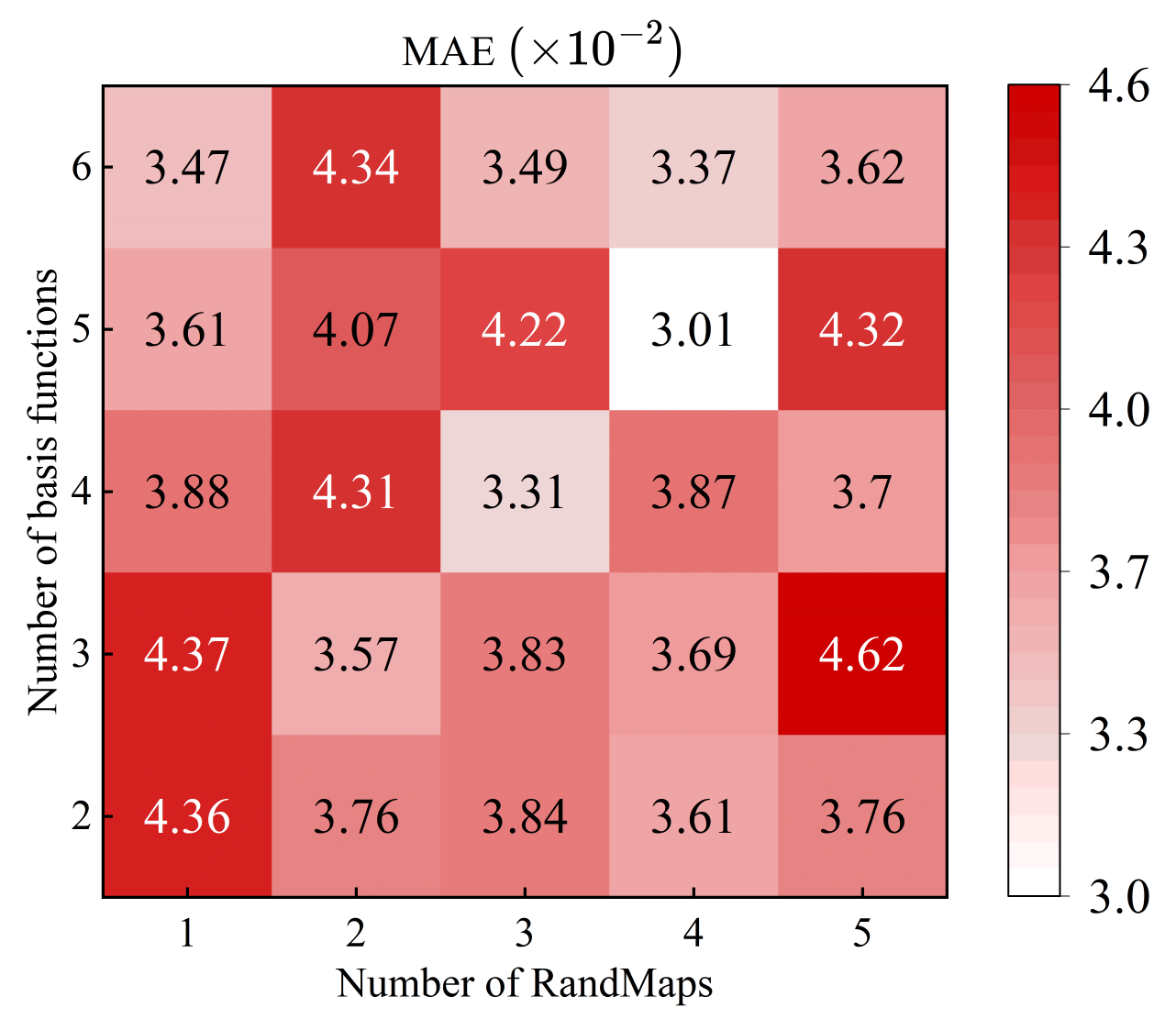}%
		\label{spring_MAE_12}}
	\hfil
	\subfloat[]{\includegraphics[width=2in]{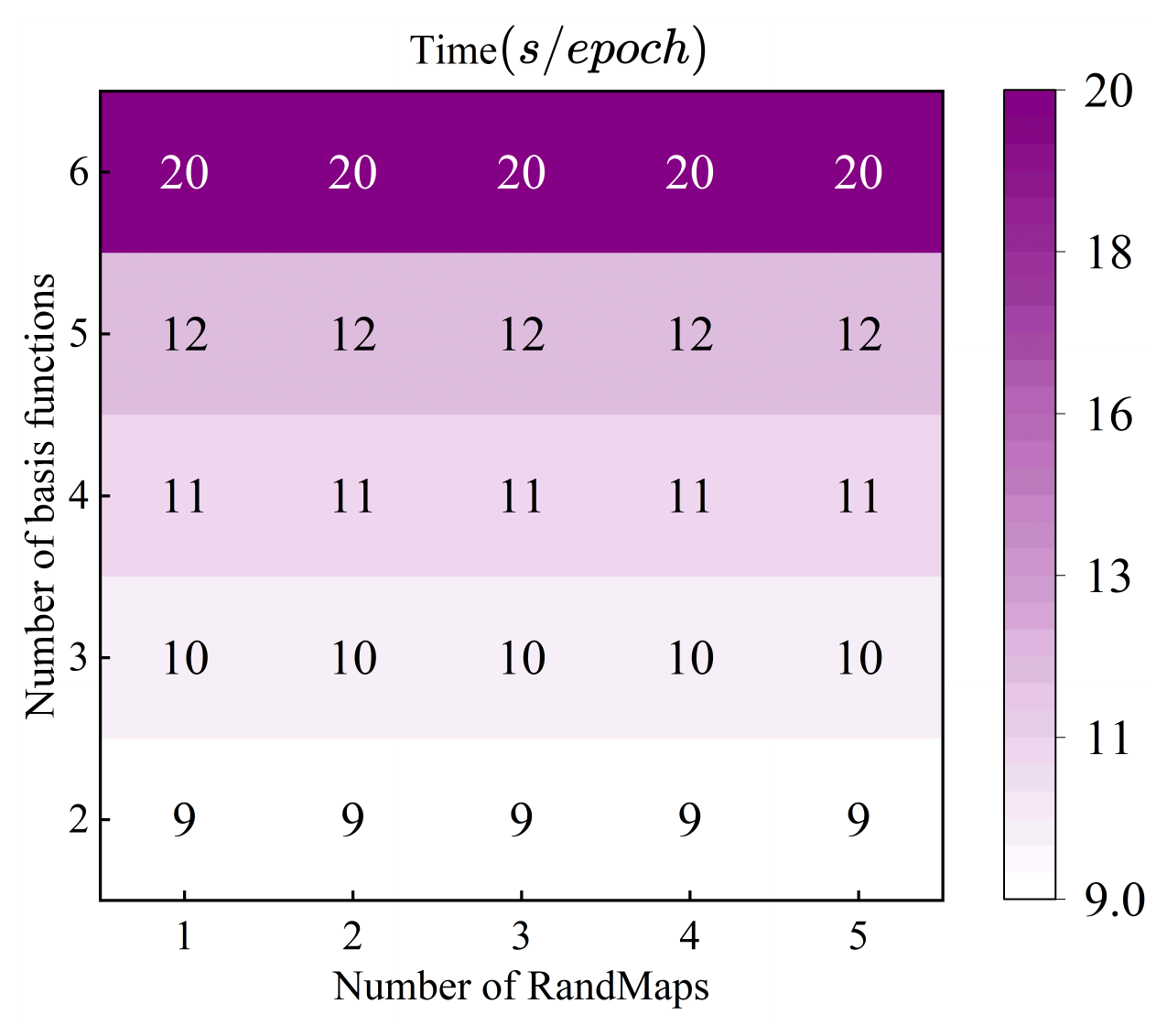}%
		\label{spring_time_12}}
	\hfil
	
	\subfloat[]{\includegraphics[width=2in]{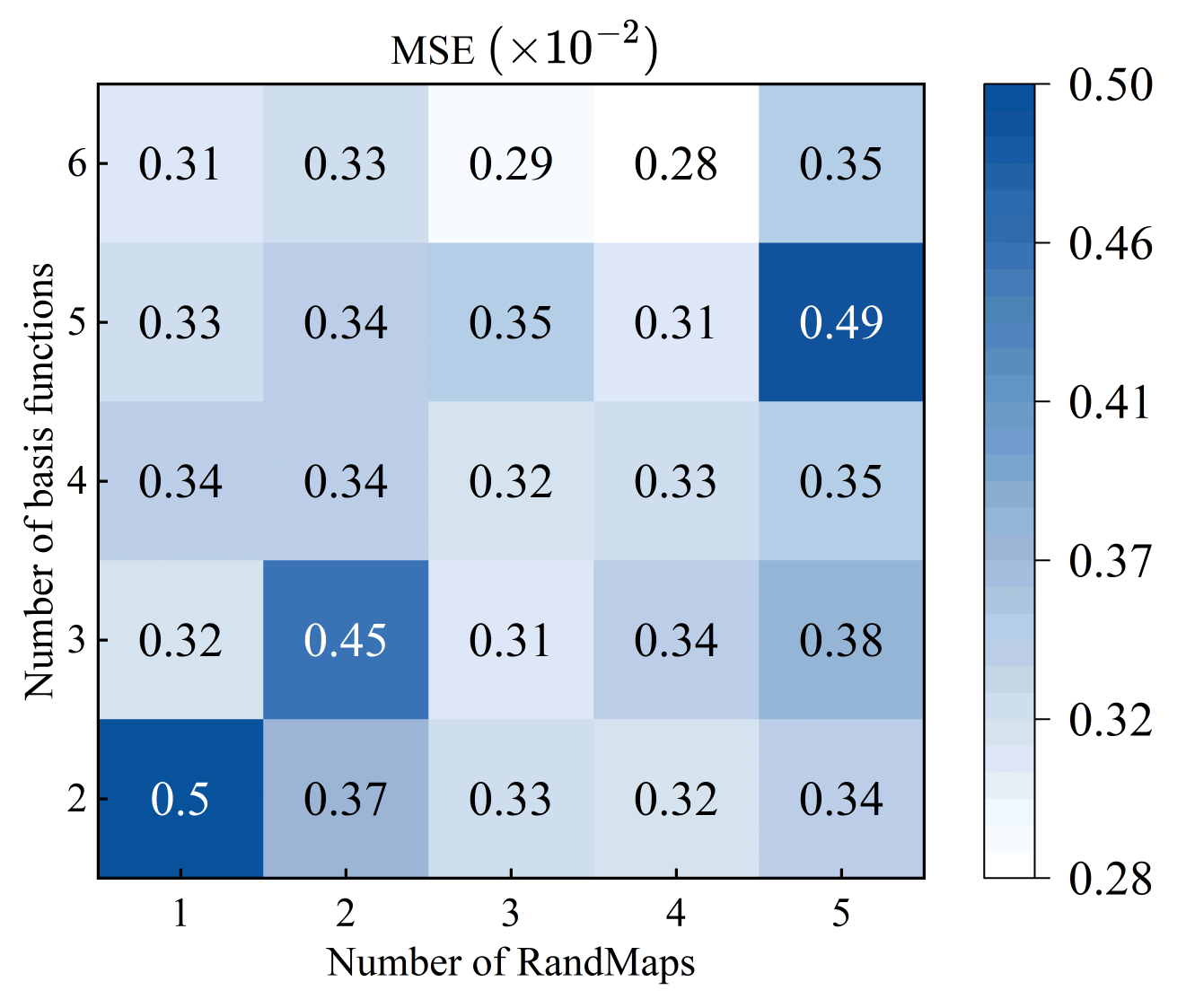}%
		\label{spring_MSE_24}}
	\hfil
	\subfloat[]{\includegraphics[width=2in]{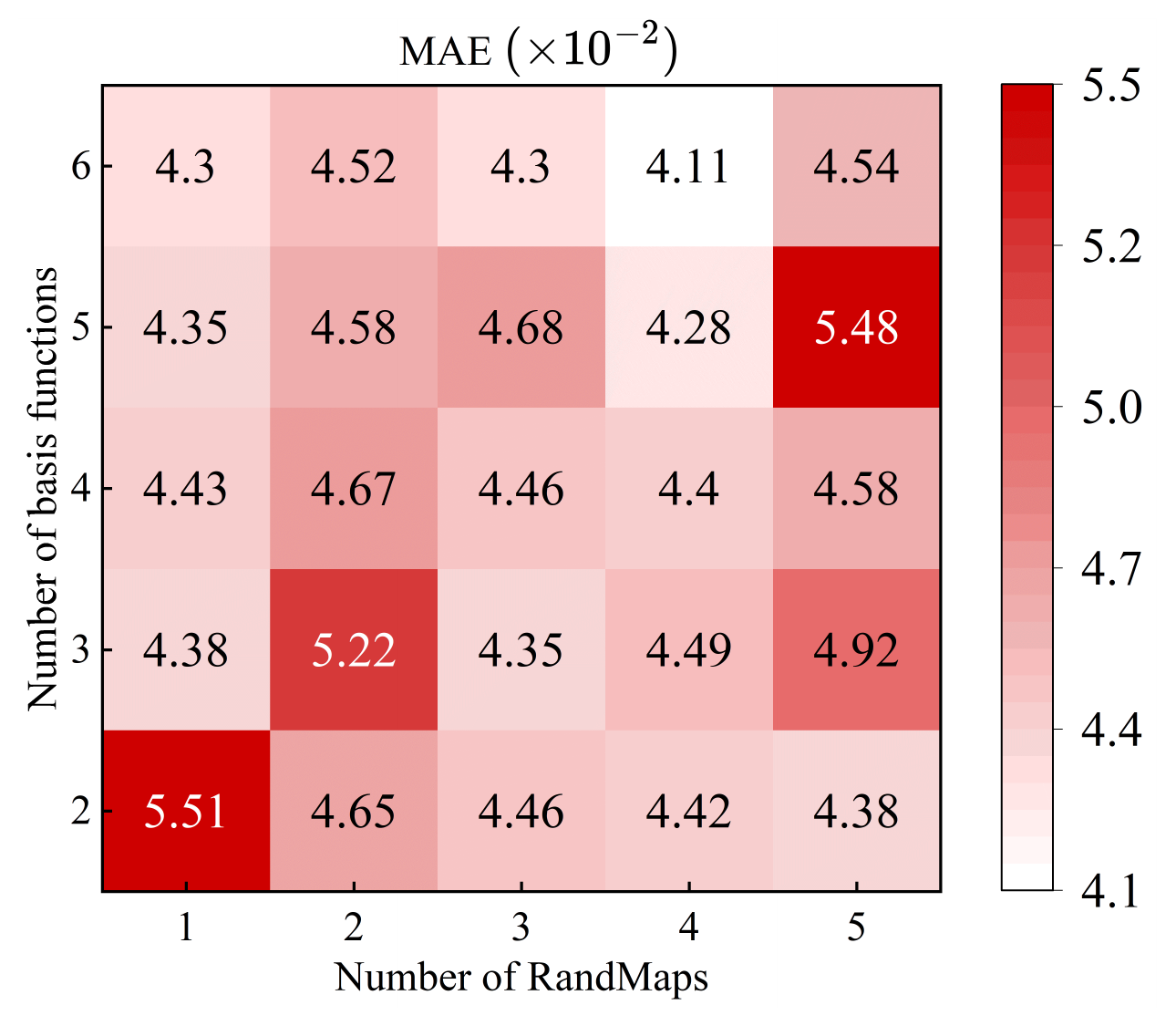}%
		\label{spring_MAE_24}}
	\hfil
	\subfloat[]{\includegraphics[width=2in]{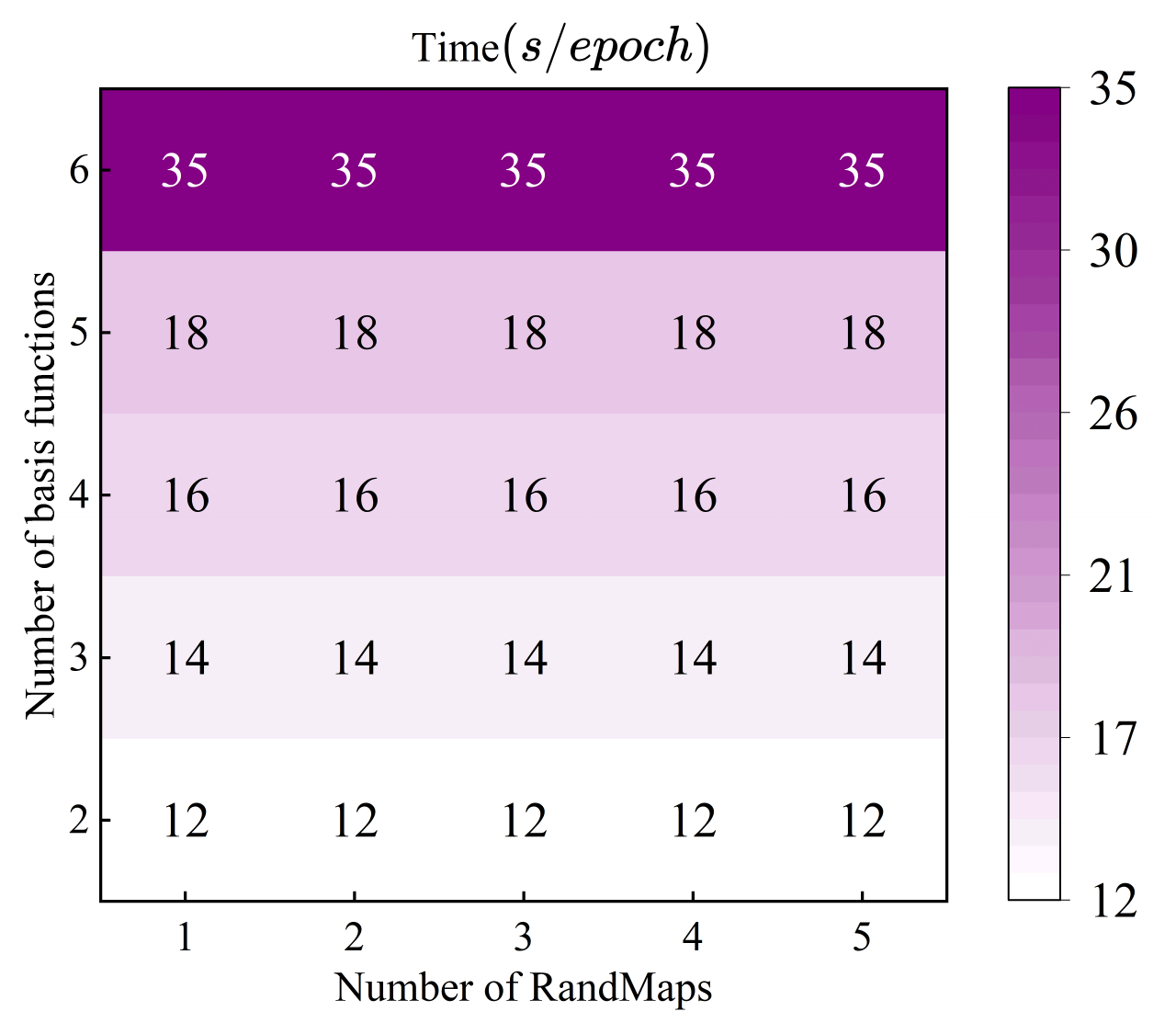}%
		\label{spring_time_24}}
	\hfil
	\caption{Sensitivity analysis on the \textit{Spring} dataset. (a) MSE results (12-step ahead). (b) MAE results (12-step ahead). (c)Training time per epoch under different configurations (12-step ahead). (d) MSE results (24-step ahead). (e) MAE results (24-step ahead). (f)Training time per epoch under different configurations (24-step ahead).}
	\label{spring}
\end{figure*}
\subsection{Implementation Details}
For the baseline methods mentioned above,we follow the hyperparameter settings used in the original
paper. We divide each sample into two parts: the first part serves as the historical trajectory, used to predict the sequence of the second part. The lengths of the two parts are defined as condition length and prediction length, where the condition length is fixed at 12 and the prediction length is set to either 12 or 24. We use Mean Squared Error (MSE) and Mean Absolute Error (MAE) as the two primary evaluation metrics. The hyperparameter configurations for the experiments are detailed in \textcolor{blue}{\cref{Hyperparameter setup}}.

\begin{figure}[htbp]
	\centering
	\subfloat[]{\includegraphics[width=1.5in]{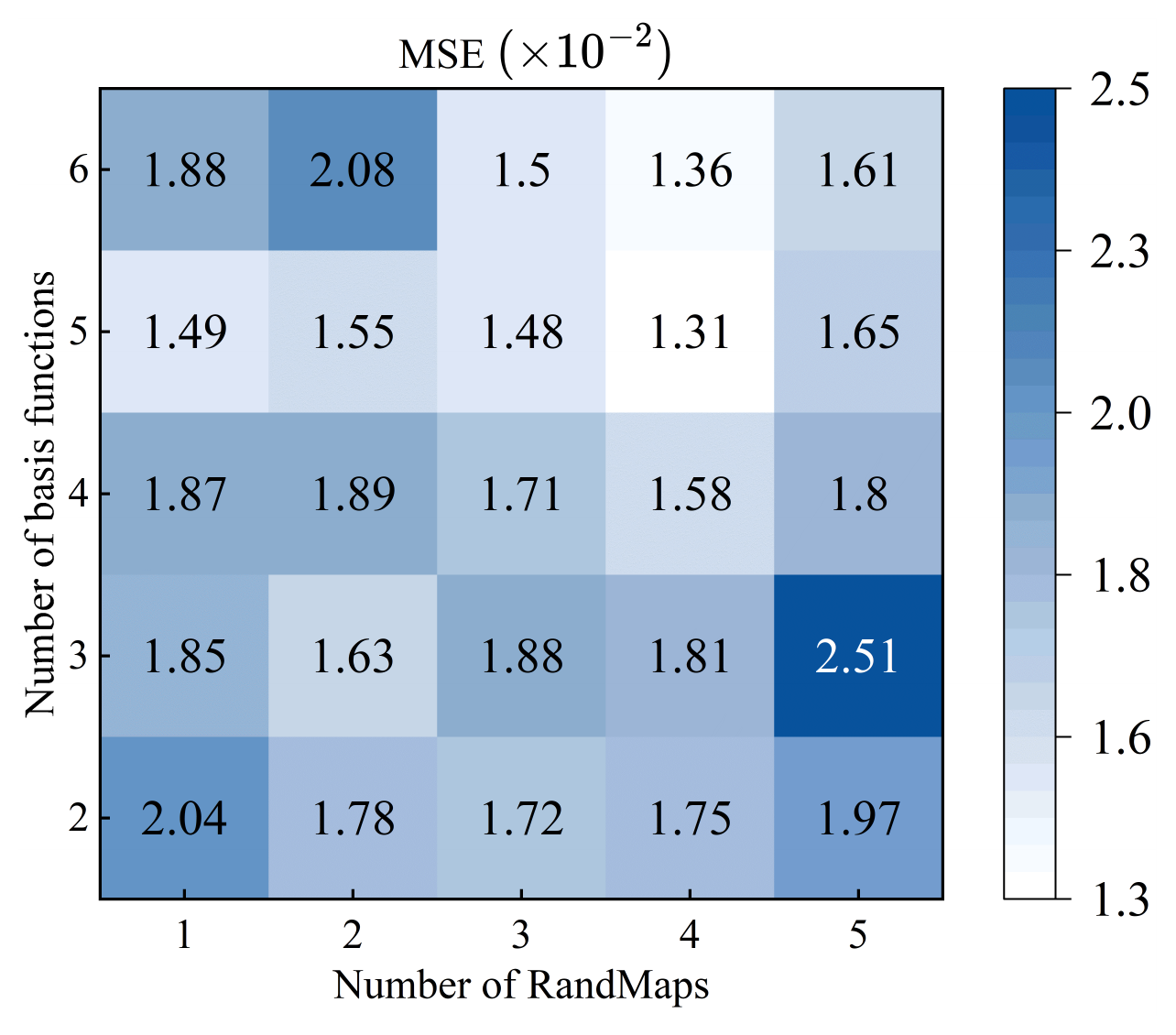}%
		\label{charged_MSE_12}}
	\hfil
	\subfloat[]{\includegraphics[width=1.5in]{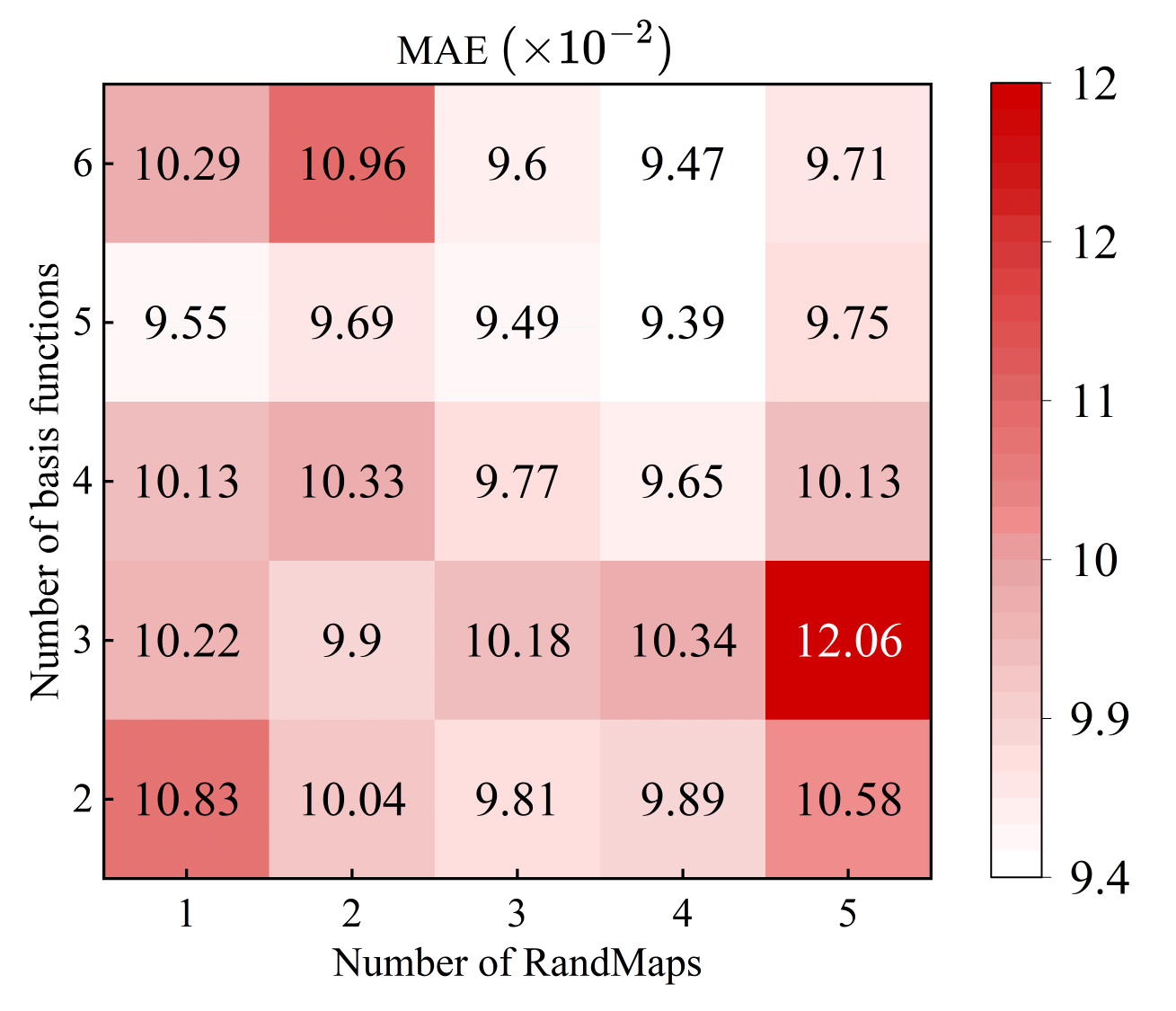}%
		\label{charged_MAE_12}}
	\hfil
	
	\subfloat[]{\includegraphics[width=1.5in]{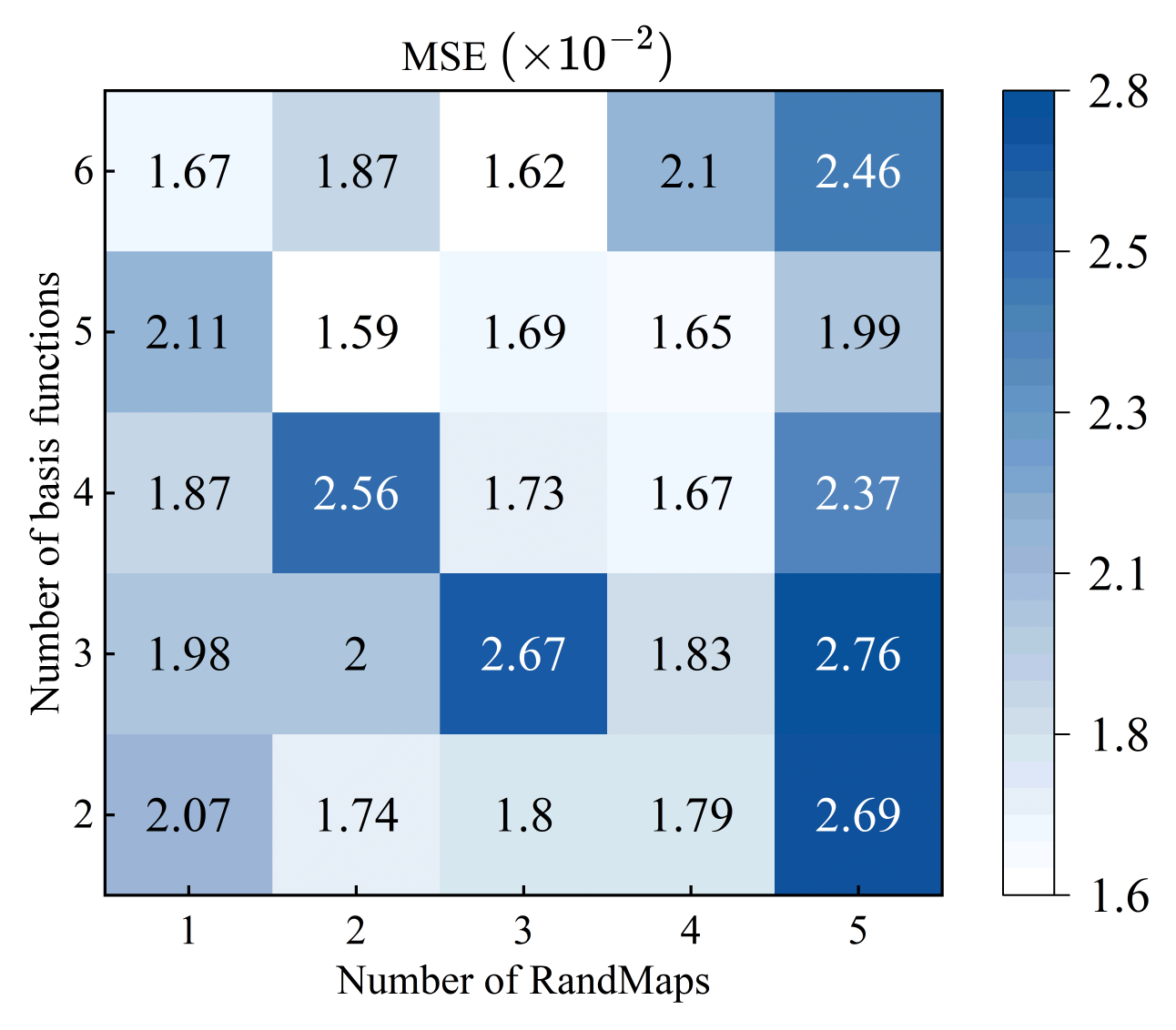}%
		\label{charged_MSE_24}}
	\hfil
	\subfloat[]{\includegraphics[width=1.5in]{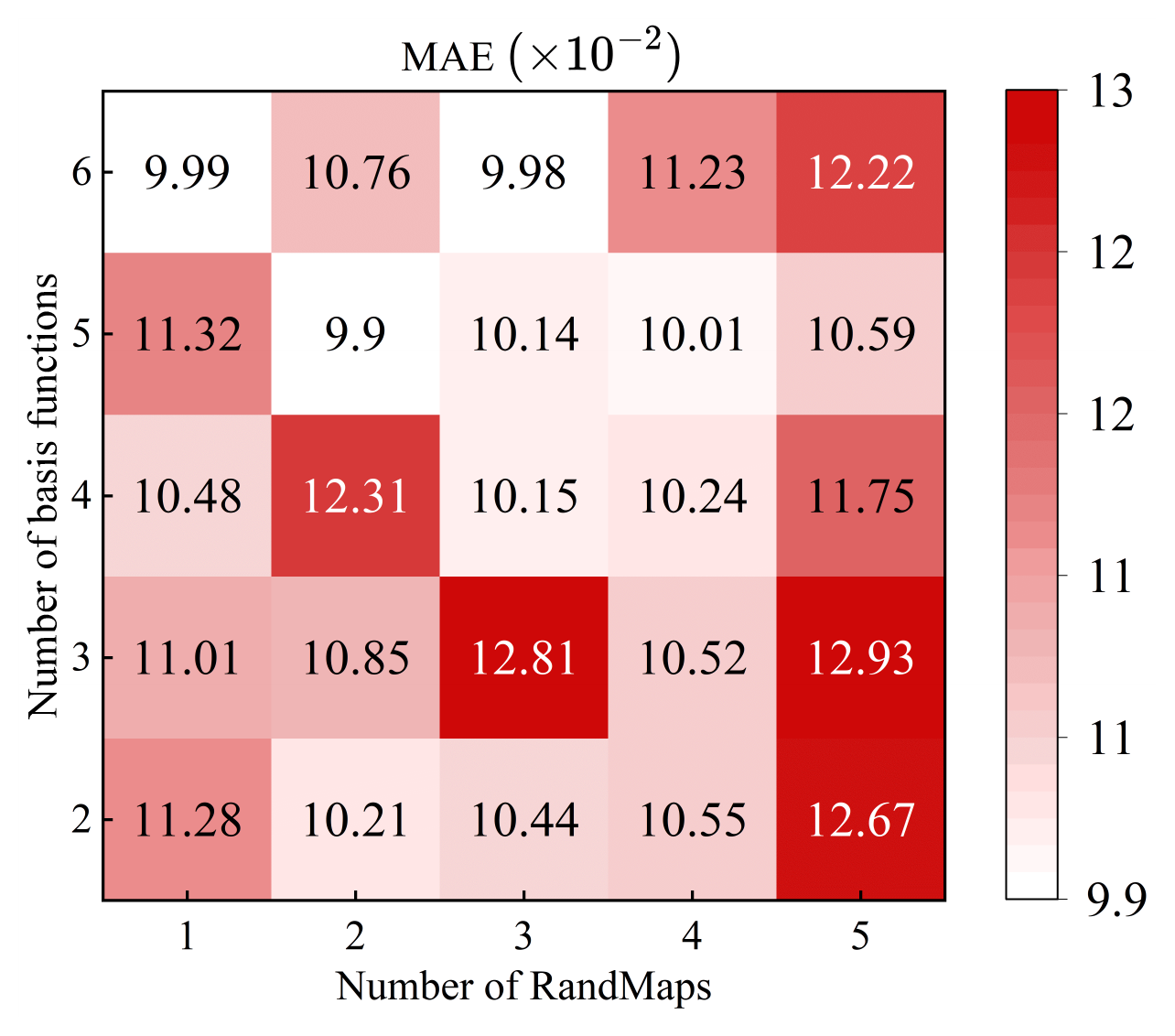}%
		\label{charged_MAE_24}}
	\hfil
	
	\caption{Sensitivity analysis on the \textit{Charged} dataset. (a) MSE results (12-step ahead). (b) MAE results (12-step ahead). (c) MSE results (24-step ahead). (d) MAE results (24-step ahead). }
	\label{charged}
\end{figure}
\subsection{Empirical Performance}
The experimental results are shown in \textcolor{blue}{\cref{compared physical}, \cref{compared molecular}} and \textcolor{blue}{\cref{compared real-world}}, where TI-ODE outperforms other baseline models in terms of MSE and MAE across six datasets. Overall, the performance of the baseline methods is limited due to their inability to effectively model inter-node dynamic interactions. Specifically, Latent-ODE does not explicitly utilize inter-node interactions; Edge-GNN, using static graph convolutions, fails to capture complex continuous-time dependencies. While established dynamic graph neural networks such as DCRNN, AGCRN, and GraphWaveNet leverage diffusion or adaptive convolutions to capture spatial correlations, they are limited in their ability to model the time-varying nature and diversity of interactions. Furthermore, despite the introduction of anisotropic message-passing mechanisms in TTS-AMP and T\&S-AMP that integrate source nodes, target nodes, and edge weights, these frameworks remain fundamentally rooted in attention-based paradigms. Since these approaches only parameterize the weighting functions within the message-passing process, they prioritize importance allocation rather than providing an explicit formulation of interaction structures and underlying dynamics. As a result, they fail to adequately characterize the diverse and time-varying nature of interactions in complex interactions. LG-ODE and CG-ODE can model latent dynamics but their interaction patterns are based on a unified message passing mechanism; CSG-ODE, while modeling node nonlinear evolution through subnetworks, still relies on a unified interaction pattern and lacks modeling of time-varying interactions; PG-ODE introduces independent interaction prototypes for each node, but its interaction pattern is shared among all neighbors and does not capture the time-varying nature of interactions. In contrast, our model explicitly captures diverse and time-evolving interaction patterns through basis functions, significantly improving the representation of dynamic relationships.

\textcolor{blue}{\cref{visualizations main}} visualizes the prediction results of the proposed TI-ODE model along with two baseline methods (CSG-ODE and PG-ODE) on the \textit{Spring} dataset. It can be observed that TI-ODE outperforms other baseline methods in trajectory reconstruction accuracy, with the predicted trajectories closely aligning with the true evolution paths. For example, when the prediction length is 24, our method accurately captures the movement direction of the cyan particle, while other models fail to identify this key dynamic feature. \textcolor{blue}{\cref{visualizations more}} provides more visualization results.
\subsection{Ablation Study}
To further analyze the components of the model, we conducted ablation experiments with three model variants: (1) TI-ODE w/o W, where the time-varying weights are removed; (2) TI-ODE w/o R, where RandNet is removed and the mean and variance are generated solely from the original sequence representations; (3) TI-ODE w/o O, where the number of basis functions is set to 1. The experimental results are shown in \textcolor{blue}{\cref{Ablation}}. Removing the time-varying weights (TI-ODE w/o W) or simplifying to a single interaction function (TI-ODE w/o O) both lead to a significant performance drop, indicating that time-varying weights and the multi-basis function structure are crucial for accurately capturing the true interaction patterns. Furthermore, removing RandNet (TI-ODE w/o R) also results in performance degradation, suggesting that RandNet provides richer initial representations in the latent space, enhancing the model's expressiveness and generalizability.
\begin{figure}[htbp]
	\centering
	\subfloat[\textit{Spring}]{\includegraphics[width=1.1in]{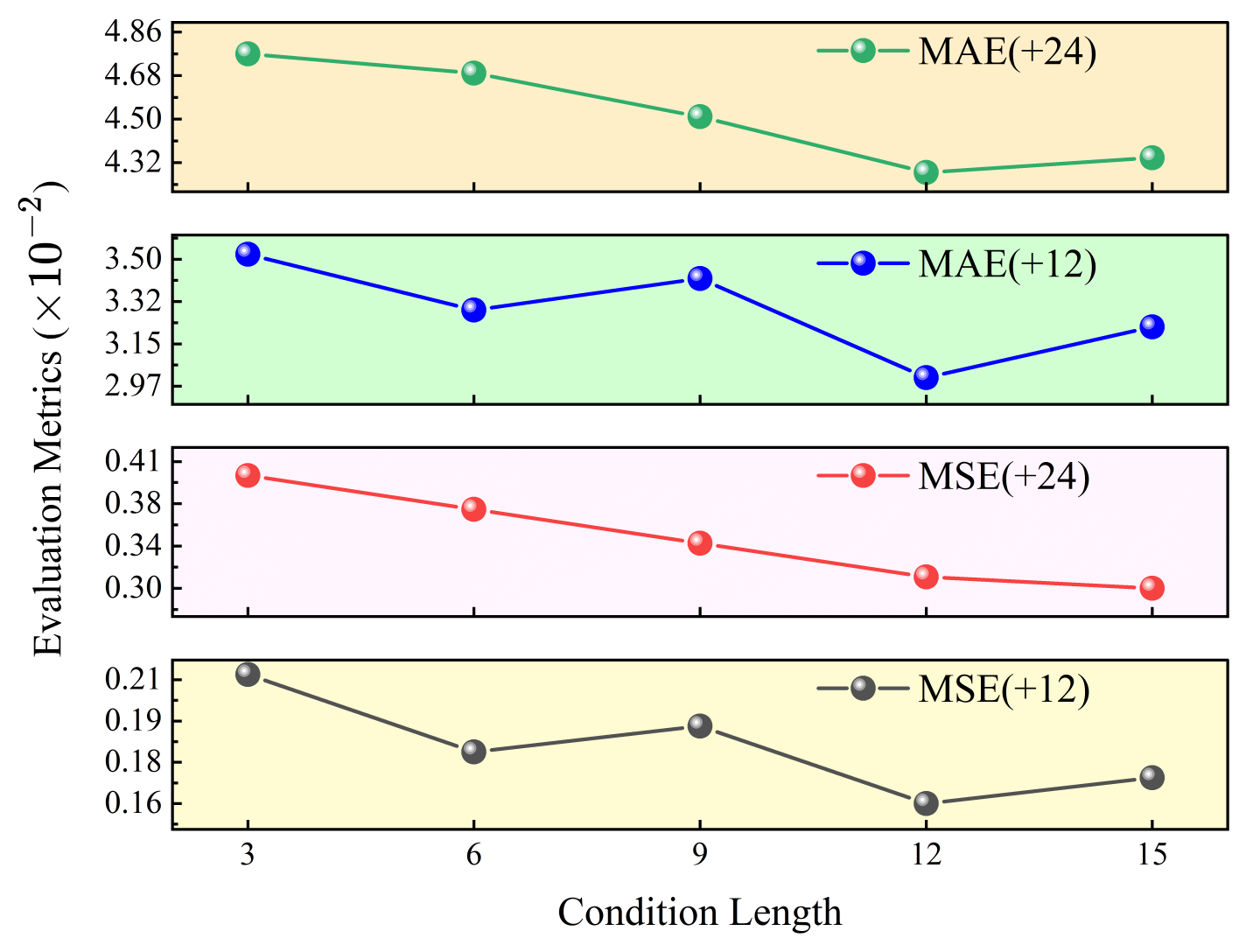}
		\label{condition_springs}}
	\hfil
	\subfloat[\textit{Charged}]{\includegraphics[width=1.1in]{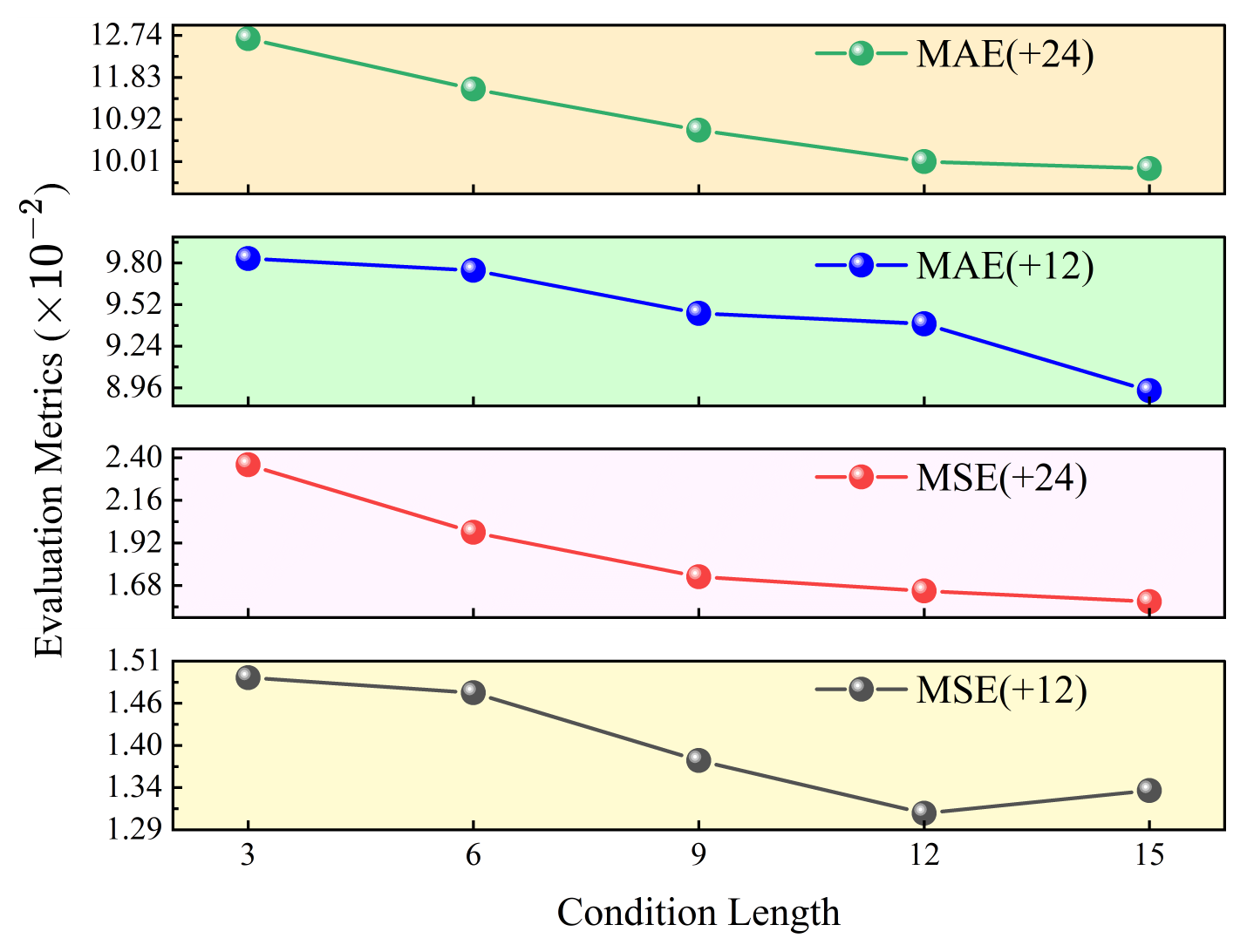}
		\label{condition_charged}}
	\hfil
	\subfloat[\textit{Motion}]{\includegraphics[width=1.1in]{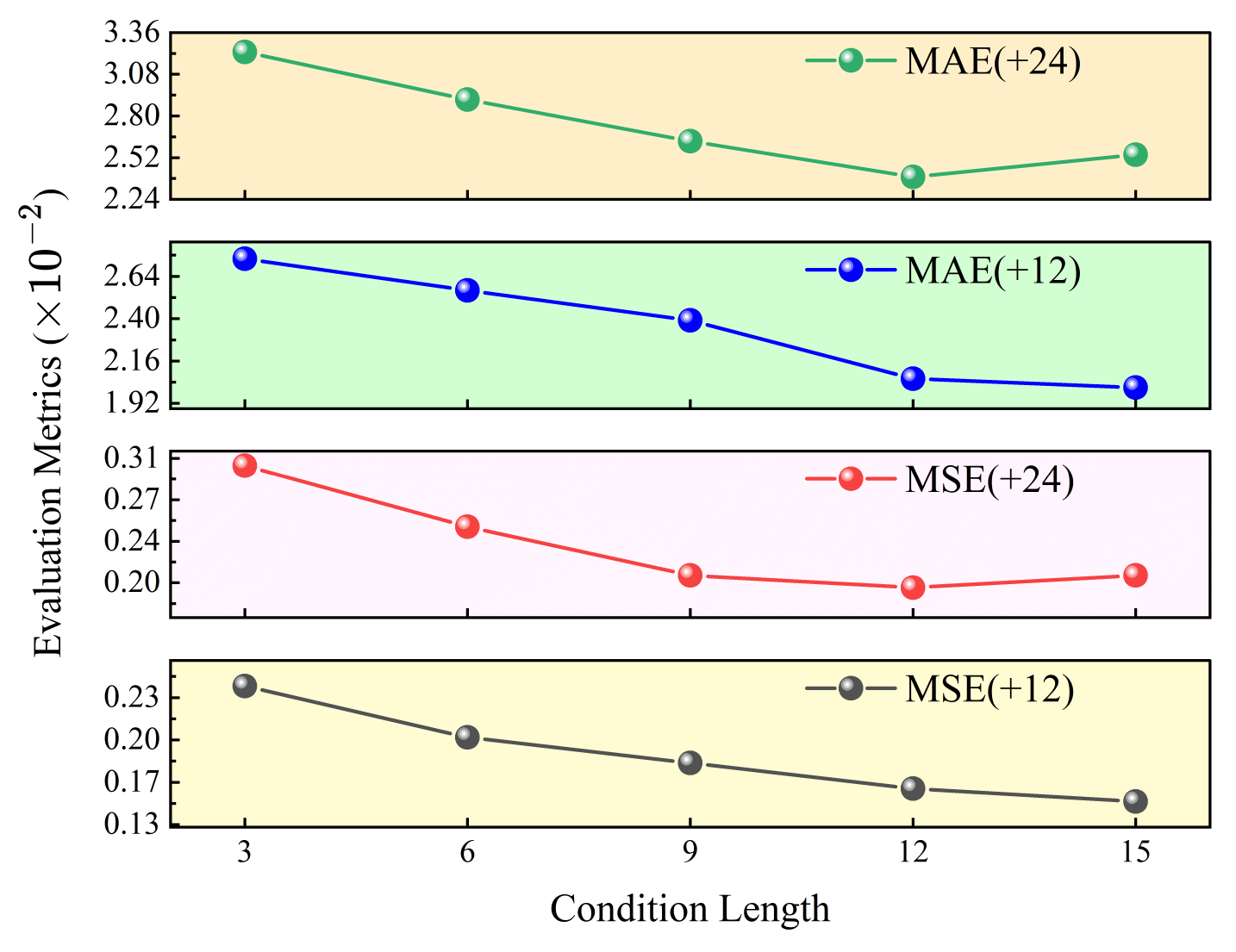}
		\label{condition_motion}}
	\hfil
	\subfloat[\textit{2N5C}]{\includegraphics[width=1.1in]{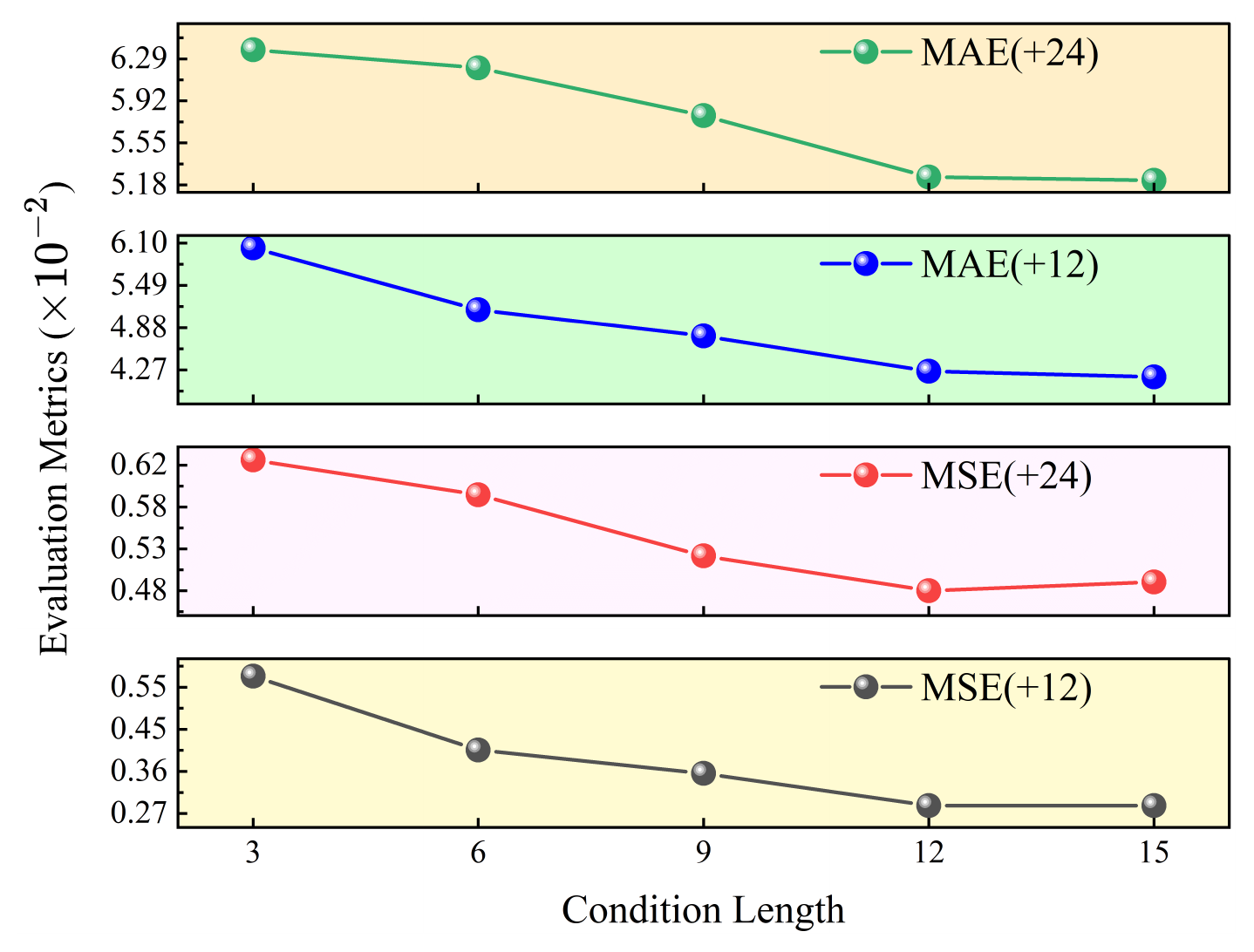}
		\label{condition_2N5C}}
	\hfil
	\subfloat[\textit{5AWL}]{\includegraphics[width=1.1in]{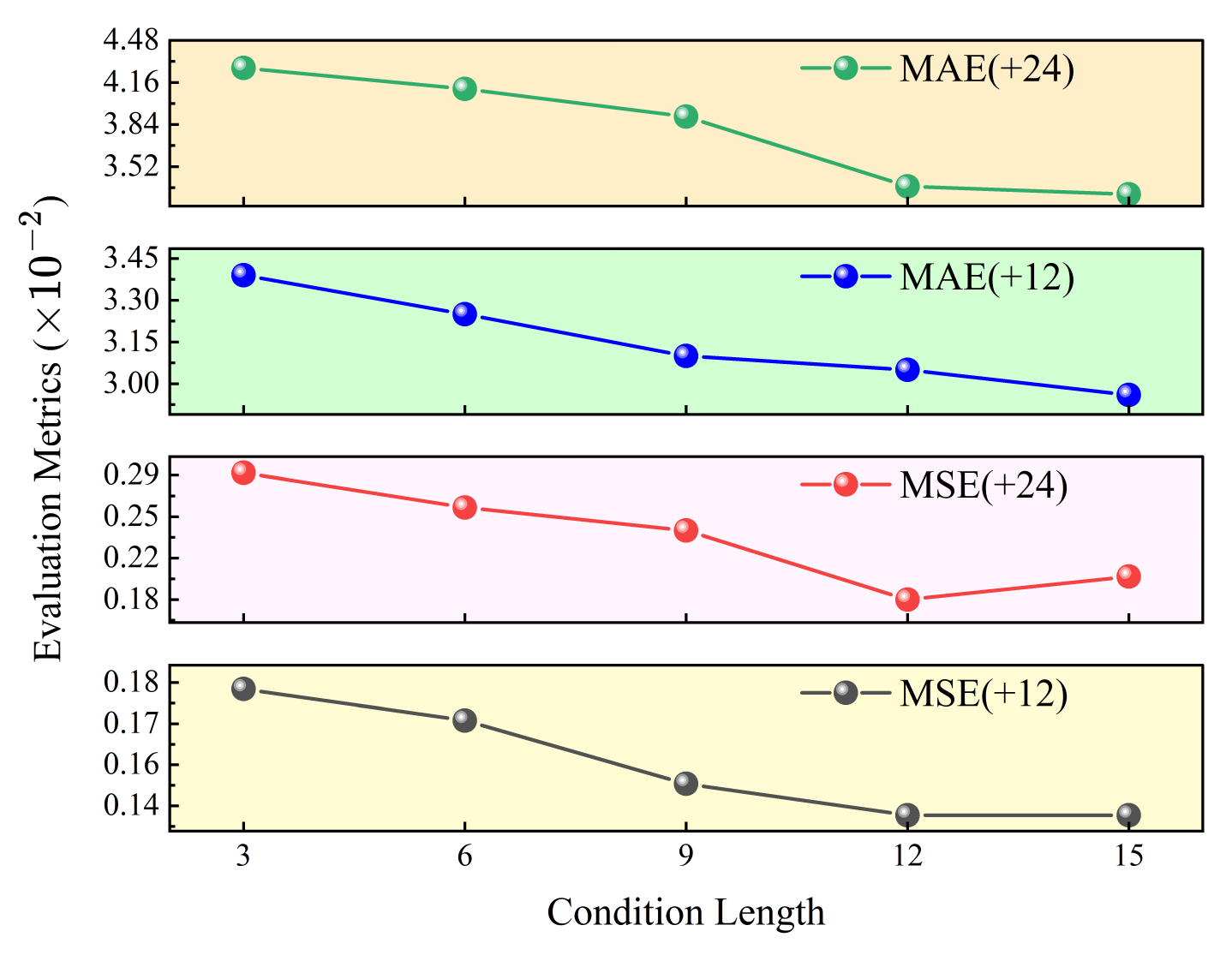}
		\label{condition_5AWL}}
	\hfil
	
	\caption{Performance under different condition lengths. }
	\label{condition}
\end{figure}
\subsection{Parameter Sensitivity}
To assess the impact of the number of basis functions and RandNet on model performance, we conducted a parameter sensitivity experiment in the combination space of basis function $K \in \{ 2,3,4,5,6\} $  and RandNet numbers ${K^r} \in \{ 1,2,3,4,5\}$. The results on the \textit{Spring} and \textit{Charged} dataset are shown in \textcolor{blue}{\cref{spring}} and \textcolor{blue}{\cref{charged}}. From the heatmaps of MSE and MAE, it can be observed that model performance improves gradually as the number of basis functions and RandNet increases, with slight fluctuations occurring after reaching a certain scale. Specifically, when the number of basis functions is 5 and the number of RandNets is 4, the model performs best in both MSE and MAE. The time consumption chart shows that increasing the number of basis functions linearly increases the computational overhead. Balancing accuracy and training overhead, we ultimately chose 5 basis functions and 4 RandNet units as the default parameter configuration for the model, achieving a good balance between expressiveness and efficiency.

To further investigate the impact of condition length on prediction performance, we conducted a sensitivity analysis within the range of $\{3,6,9,12,15\}$ and tested the model’s performance under two prediction lengths. The results shown in \textcolor{blue}{\cref{condition}}. Shorter condition lengths (such as 3 or 6) generally lead to higher MSE and MAE because the model fails to capture sufficient early information, resulting in unstable initial state estimates. As the condition length increases to 9 and 12, model performance improves significantly, with errors steadily decreasing, especially in long-term prediction tasks. When the condition length increases from 12 to 15, the performance improvements tend to plateau, with slight fluctuations, indicating that excessively long condition sequences offer limited gains and may introduce redundant information. Considering both short-term and long-term prediction task performances, we ultimately chose a condition length of 12 as the optimal historical information window size, balancing information sufficiency and training efficiency.

\subsection{More Experimental Results}
\begin{figure*}[!ht]
	\centering
	\subfloat[]{\includegraphics[width=2in]{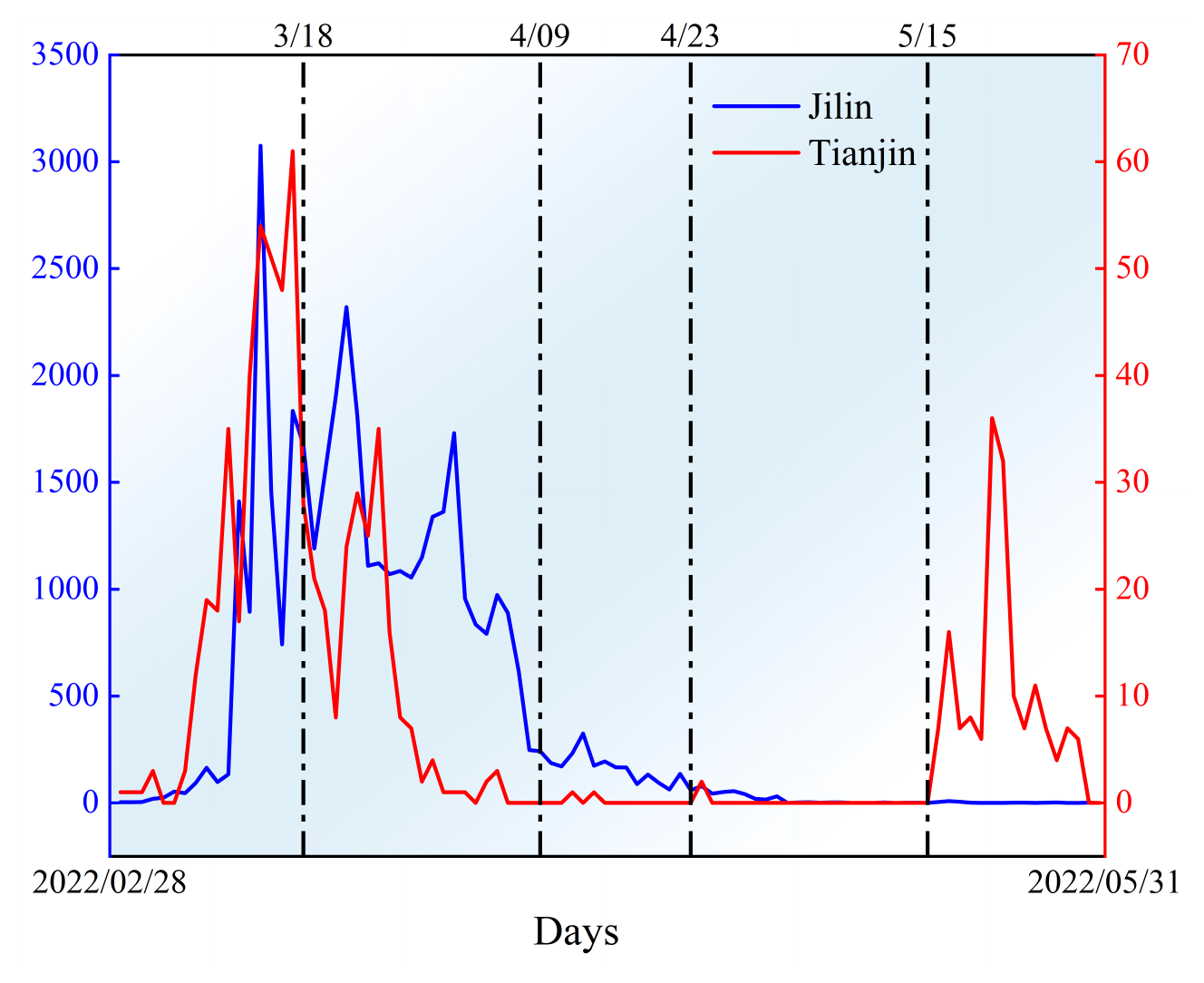}%
		\label{bias weight a}}
	\hfil
	\subfloat[]{\includegraphics[width=2in]{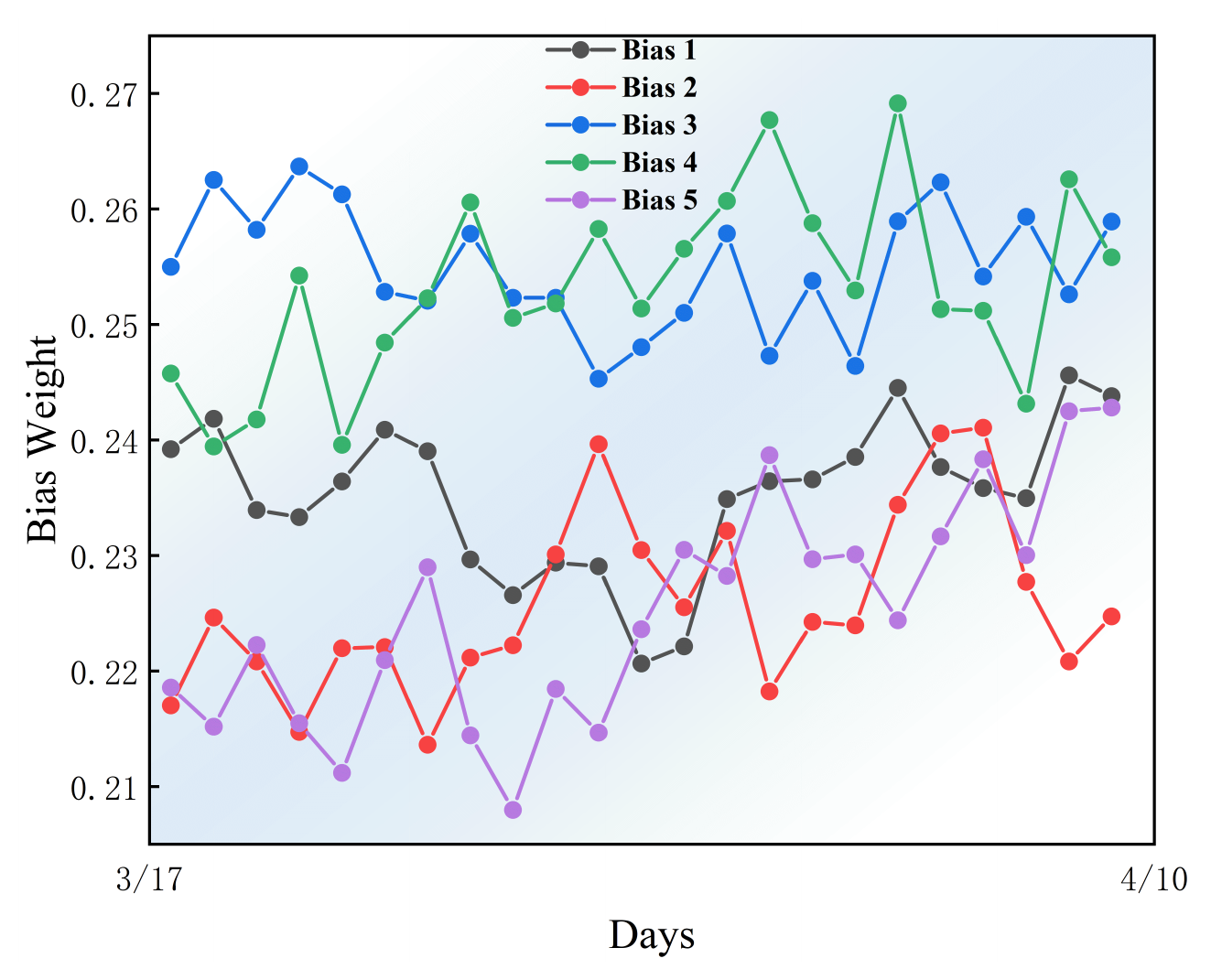}%
		\label{bias weight b}}
	\hfil
	\subfloat[]{\includegraphics[width=2in]{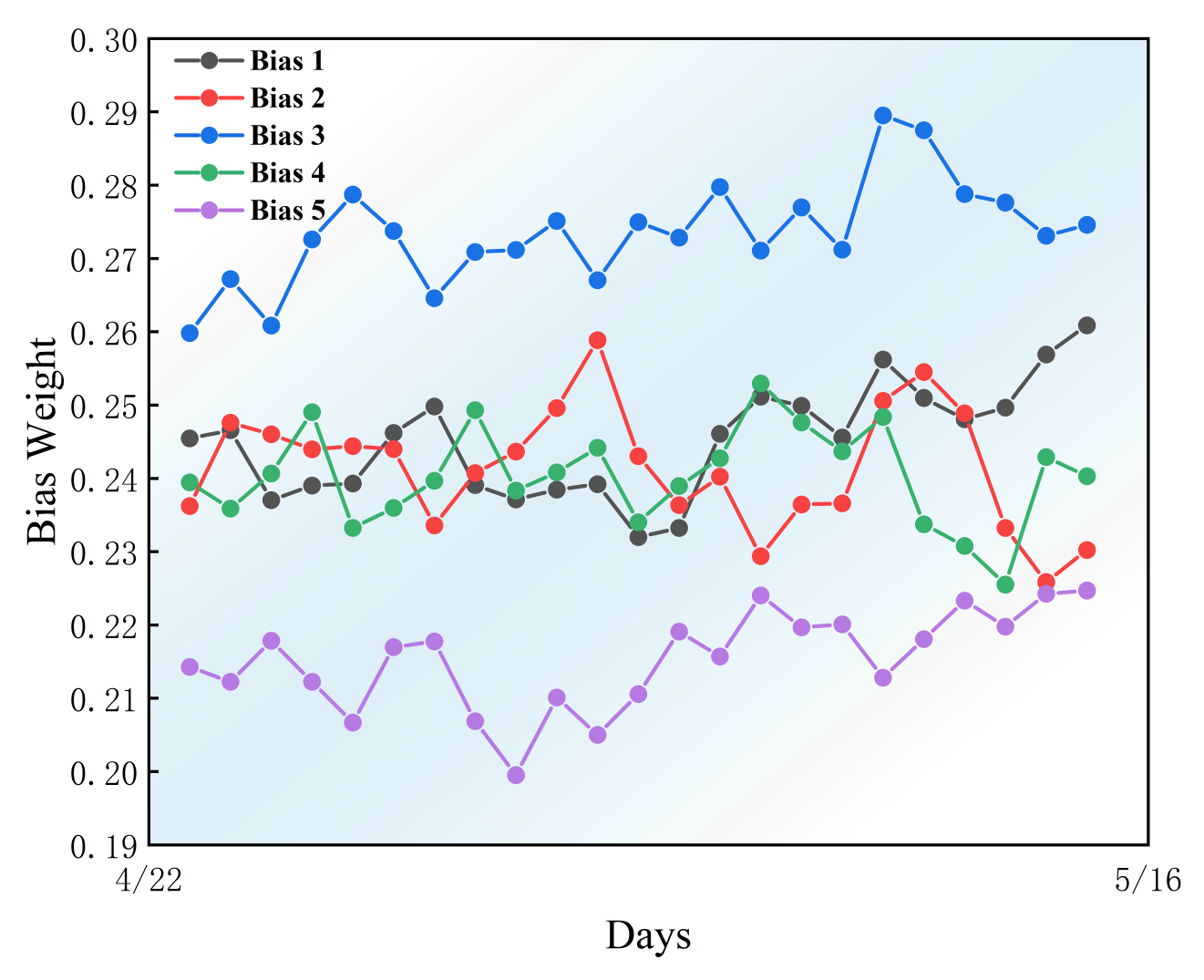}%
		\label{bias weight c}}
	\hfil
	
	\caption{Visualization of the learned bias weights. (a) Daily new confirmed COVID-19 cases in Jilin and Tianjin from March 1 to May 31, 2022. (b) Weight visualization for the period from March 18 to April 9, 2022. (c) Weight visualization for the period from April 23 to May 15, 2022}
	\label{fig_sim}
\end{figure*}
\subsubsection{Bias Weight Visualizations}

To validate the plausibility of the basis function weights learned by the model, we conducted a visualization analysis on two periods in the COVID-19 dataset: March 18 to April 9, 2022, and April 23 to May 15, 2022. The analysis focuses on the temporal evolution of all basis function weights $w_{i,j,k}^t$ from Jilin to Tianjin. \textcolor{blue}{\cref{bias weight a}} shows the daily confirmed cases in Jilin and Tianjin from March 1 to May 31, 2022, while \textcolor{blue}{\cref{bias weight b}} and \textcolor{blue}{\cref{bias weight c}} present the temporal trajectories of all basis function weights $w_{i,j,k}^t$ between the two regions at different time points. Here, the weight values quantify the strength of the corresponding basis functions, with larger weights indicating a more significant contribution to inter-node interactions at a given time.

According to publicly released epidemic reports, Jilin Province entered a phase of large-scale outbreak in mid-March, with the number of confirmed cases rapidly increasing, followed by the implementation of strict lockdown and transportation restriction measures. In contrast, the epidemic situation in Tianjin remained relatively stable, while inter-provincial population mobility was significantly affected by the lockdown policies in Jilin. As shown in \textcolor{blue}{\cref{bias weight a}} and \textcolor{blue}{\cref{bias weight b}}, the weights of different basis functions exhibit distinct temporal evolution patterns. For instance, the weight of Basis Function 1 drops sharply in late March and gradually recovers in early April. The interaction pattern captured by this basis function exhibits a highly consistent temporal trend with changes in inter-provincial population mobility, and thus can be interpreted as a transmission mechanism associated with cross-provincial travel. When control measures become more stringent and population mobility is restricted, the corresponding weight decreases; conversely, as lockdown policies are gradually relaxed and cross-regional mobility recovers, the weight increases accordingly. These observations indicate that the interaction weights learned by the model can reflect the influence of macro-level policies and population mobility dynamics on the underlying transmission structure. During the period from April 23 to May 15, 2022, Jilin entered a phase of epidemic mitigation and gradual resumption of social activities. \textcolor{blue}{\cref{bias weight a}} and \textcolor{blue}{\cref{bias weight c}} show that the fluctuations of basis function weights are markedly smaller than in the previous period, and the relative influence of different transmission mechanisms becomes more balanced. Compared with the outbreak phase in March, the weight trajectories in this stage are more stable, indicating that the model effectively captures the temporal characteristics of the epidemic as it transitions to normalized control. This further demonstrates the effectiveness of our method in characterizing time-varying interaction mechanisms in real-world systems.
\subsubsection{Training Times}
\begin{figure}[htbp]
	\centering
	\subfloat[12-step ahead]{\includegraphics[width=2in]{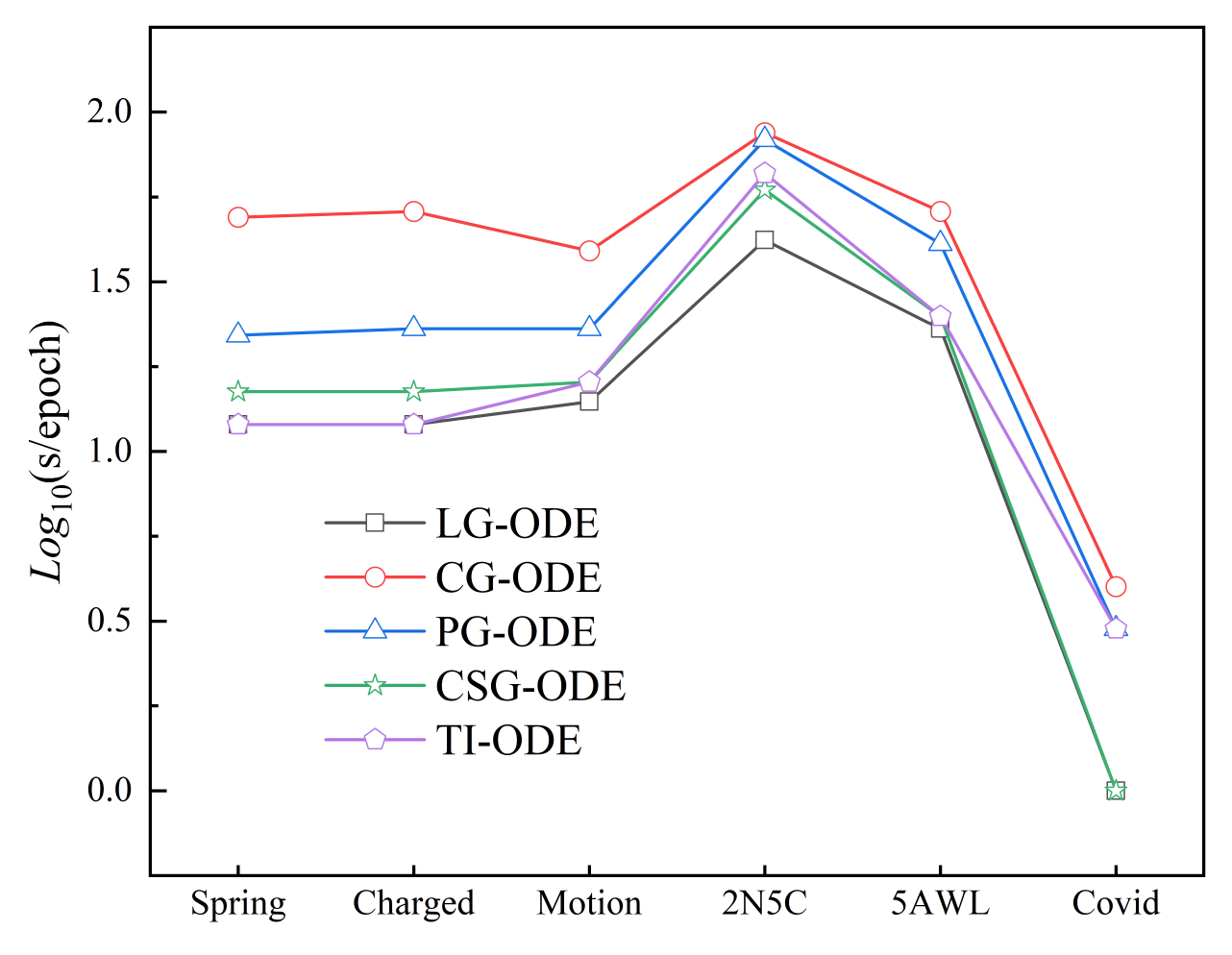}%
		\label{run_time_12}}
	\hfil
	\subfloat[24-step ahead]{\includegraphics[width=2in]{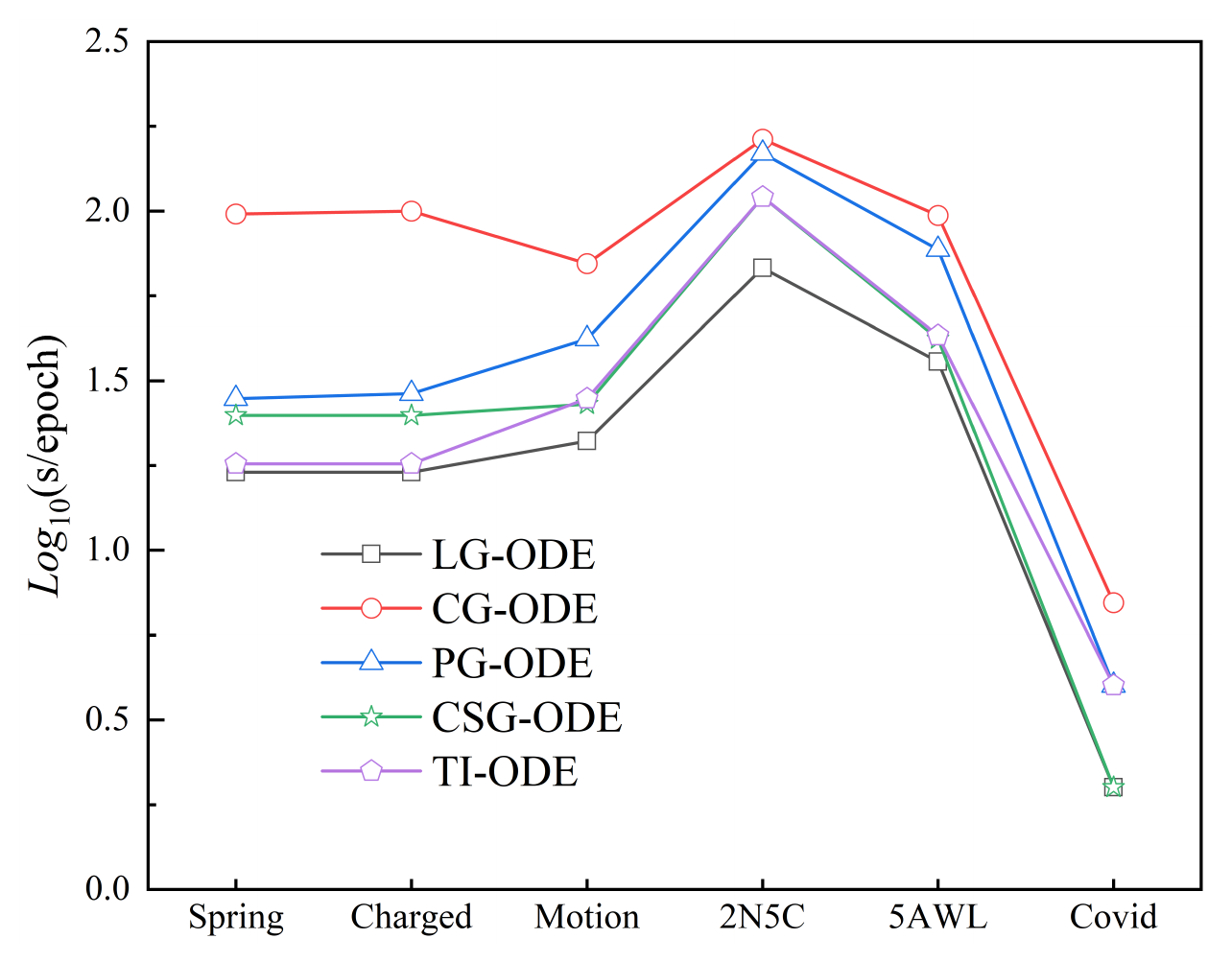}%
		\label{run_time_24}}
	\hfil
	
	\caption{Comparison of training times for different methods on
		the all datasets. The training durations are presented in units of $log_{10}$(seconds/epoch). }
	\label{run_time}
\end{figure}
\begin{figure}[htbp]
	\begin{center}
		\includegraphics[width=3in]{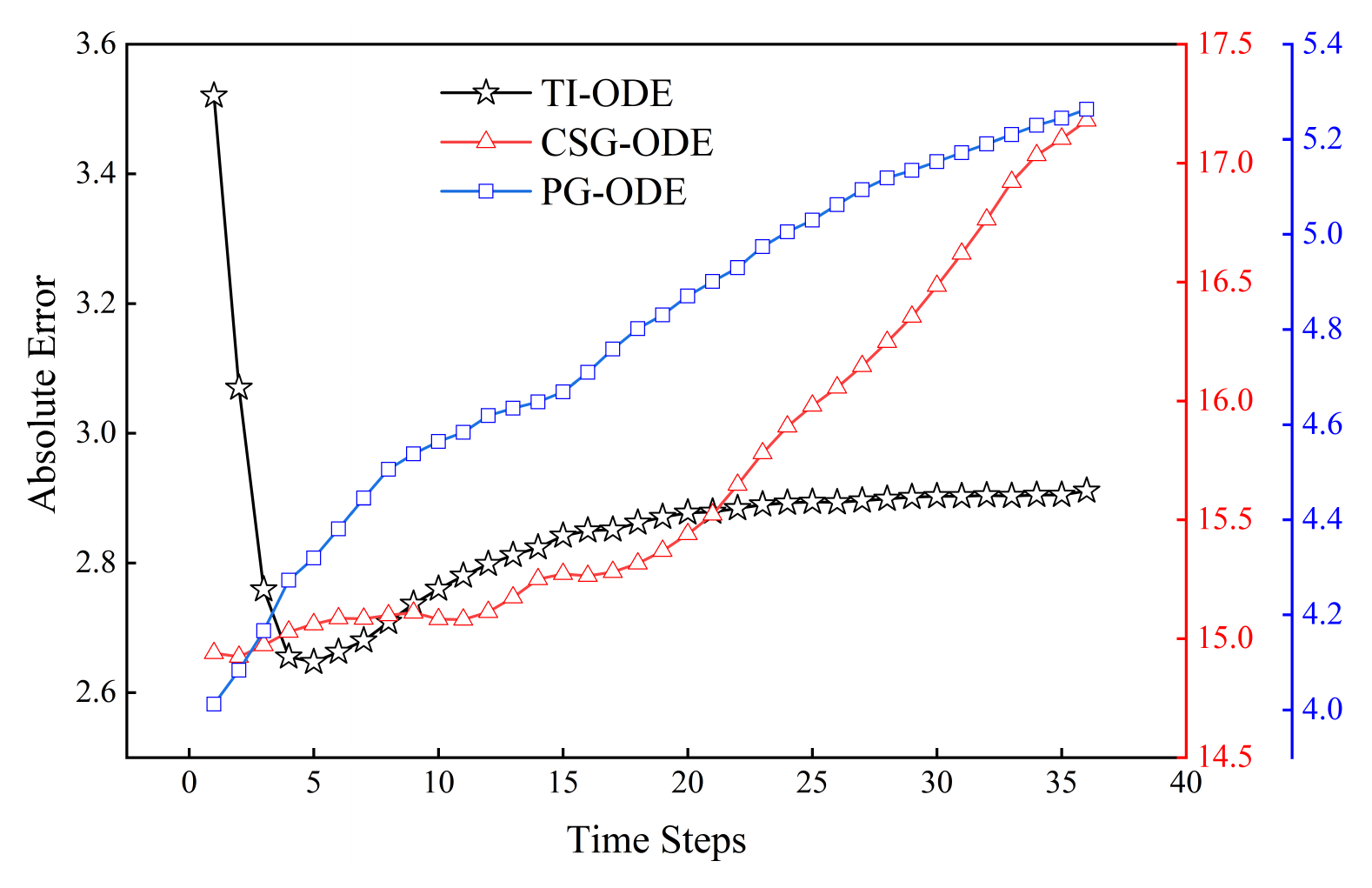}
	\end{center}
	\caption{Temporal Evolution of Absolute Error under Initial State Perturbations across Different Models.}
	\label{robustness experience fig}
\end{figure}
To validate the computational efficiency of TI-ODE, we compared the training time of TI-ODE
with other graph neural ODE methods (LG-ODE, CG-ODE, PG-ODE and CSG-ODE). All experiments were run on a Titan RTX GPU, and we compared the time cost per training epoch. As shown in \textcolor{blue}{\cref{run_time}}, although TI-ODE introduces multiple interaction basis functions and increases the parameter count, its impact on the overall training time is negligible. Compared to other graph neural ODE methods such as CSG-ODE and PG-ODE, TI-ODE achieves superior property prediction performance while maintaining comparable training efficiency. These results indicate that TI-ODE delivers significant performance gains while sustaining computational efficiency.

\subsubsection{Model Robustness Experiment}
\begin{figure}[htbp]
	\begin{center}
		\includegraphics[width=3in]{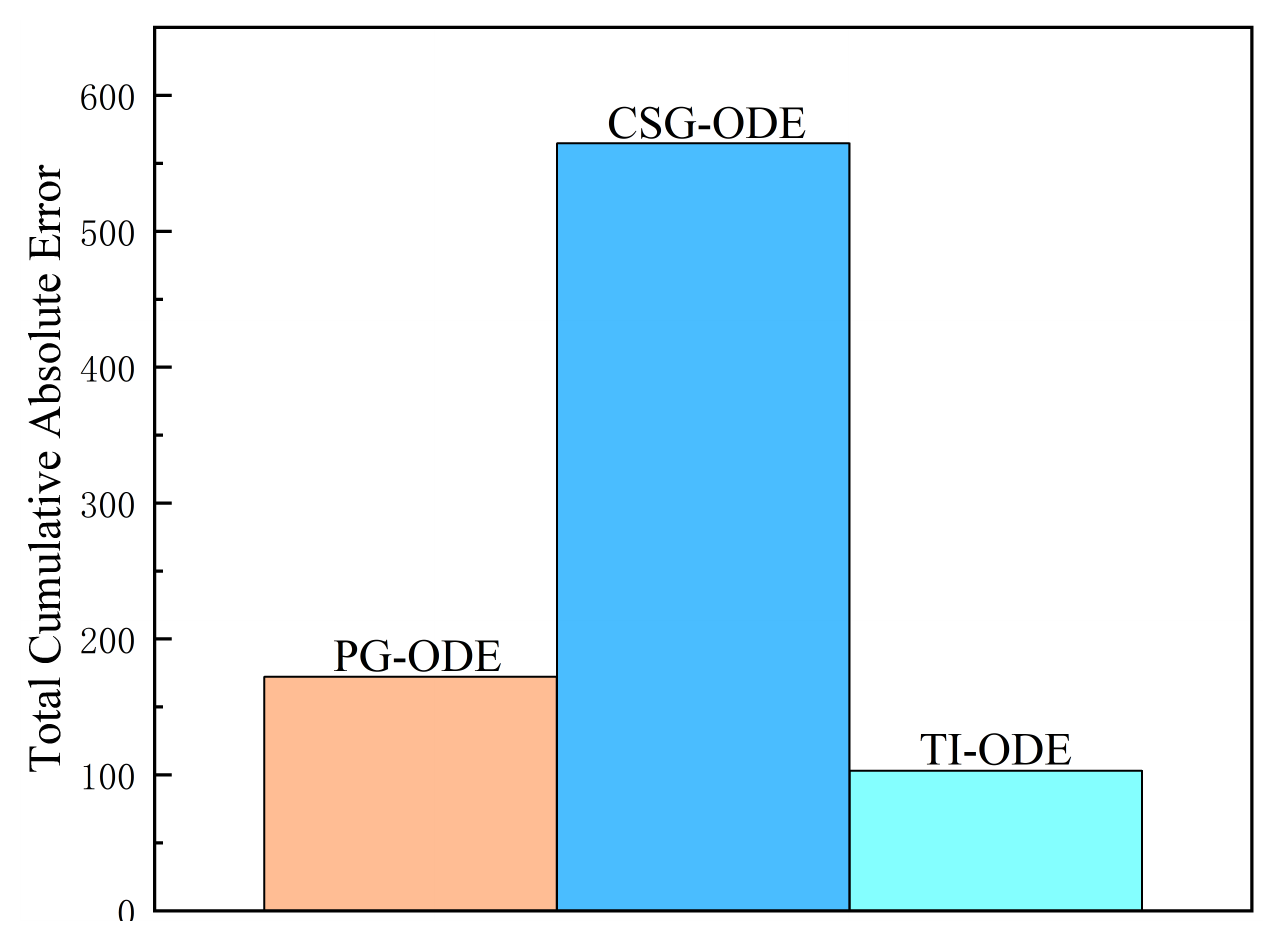}
	\end{center}
	\caption{Comparison of Total Cumulative Error across Different Models.}
	\label{cum ae}
\end{figure}

We investigate whether TI-ODE exhibits a lower perturbation growth rate and smaller cumulative error compared to models with a unified message passing mechanism, such as CSG-ODE and PG-ODE. The core logic of this experiment is to validate model robustness by introducing infinitesimal perturbations to the initial state and observing the resulting cumulative error during long-term evolution. For dataset selection, we utilize \textit{Spring}, a physical simulation dataset with a well-defined dynamical structure and explicit analytical laws, making it ideal for monitoring deviations over extended horizons. We trained TI-ODE, CSG-ODE, and PG-ODE with both the condition length and prediction length set to 12. During inference, a small noise (fixed at 0.01) is applied to the initial values to obtain perturbed initial states. We then extend the prediction length to 36 steps, using the same ODE solver (e.g., RK4) to evolve both the original trajectory $z_i^t$ (from unperturbed states) and the perturbed trajectory $\tilde{z}_i^t$. The absolute perturbation error $e_i^t = |\tilde z_i^t - z_i^t|$ is calculated at each timestamp to plot its variation over time $t$, and the total cumulative error over the long horizon is reported.

\textcolor{blue}{\cref{robustness experience fig}} illustrates the temporal evolution of the absolute error $e^t$ under initial state perturbations. While all models are subjected to the same magnitude of noise, their error accumulation patterns diverge significantly. Both CSG-ODE and PG-ODE exhibit a consistent and rapid upward trend in absolute error as the prediction horizon extends, indicating sensitivity to initial disturbances and a tendency for error propagation over time. In contrast, TI-ODE demonstrates superior numerical stability. After a brief initial adjustment phase, its error curve flattens and remains remarkably stable throughout the 36-step evolution.
The aggregate impact of error propagation is further quantified in \textcolor{blue}{\cref{cum ae}}, which displays the total cumulative absolute error for each model. The disparity is striking: CSG-ODE suffers from the highest error accumulation, followed by PG-ODE. In contrast, TI-ODE consistently maintains the smallest total error, outperforming the closest baseline by a wide margin. These results demonstrate that TI-ODE exhibits a lower error growth rate, rendering it more robust for long-term forecasting.

\section{Conclusion}
This paper proposes TI-ODE, a graph neural ODE model that captures the diversity and time-varying nature of inter-node interactions in continuous time. TI-ODE decomposes the evolution function of a graph ODE into a set of learnable basis functions, where each basis corresponds to a distinct interaction type, and dynamically combines them via time-dependent learnable weights to precisely characterize dynamic graph evolution. Theoretically, we show that TI-ODE is more robust than models that rely on a unified message passing function. Empirically, TI-ODE consistently achieves superior performance across a diverse set of benchmarks, including physical dynamics, molecular dynamics, and real-world datasets. Its effectiveness on the COVID dataset further demonstrates strong interpretability and generalization capability. Moreover, TI-ODE exhibits improved robustness compared to models based on a unified message-passing paradigm.

While TI-ODE excels in capturing diverse and time-varying interactions between nodes, it still encounters scalability constraints when applied to large-scale dynamic graphs. Specifically, training efficiency and computational speed can be hindered when the number of nodes or edges reaches the million scale. To address these challenges in large-scale scenarios, further investigation into approximation methods, sparse representations, or more efficient training strategies is required. Future research could explore the potential of TI-ODE for both large-scale dynamic graphs and multimodal graph data.

\section*{Acknowledgments}
This work is supported by the National Natural Science Foundation of
China (Nos. 62272285, 62376142, U25A20529).
\appendix
\section{Proof of \textcolor{blue}{\cref{Existence unique}}}

\label{theorem1}

\begin{definition}
	\label{definition1}
	A function $f:\mathbb{R}^n \to \mathbb{R}^m$ is Lipschitz continuous if there exists a constant $L > 0$ such that $||f(x) - f(y)|| \le L\, ||x - y||$ for all $x, y \in \mathbb{R}^n$.
\end{definition}

\begin{remark}	 
	Unless otherwise specified, all vector norms in this paper refer to the Euclidean norm (the $l_2$ norm: $||x|| = (\sum_{i=1}^d x_i^2)^{1/2}$), and all matrix norms refer to the spectral norm ($||A||_2 = \max_{X \ne 0} \frac{||AX||_2}{||X||_2} = \sqrt{\lambda_{\max}(A^\top A)} = \sqrt{\max_{1 \le i \le n} |\lambda_i|}$, where $\lambda_i$ are the eigenvalues of $A^\top A$ and $\lambda_{\max}$ denotes its largest eigenvalue). Additionally, the vector $\infty$-norm is defined as $||x||_\infty = \max_i |x_i|$, and the matrix $\infty$-norm is defined as $||A||_\infty = \max_i \sum_{j=1}^m |a_{ij}|$, which corresponds to the maximum $\ell_1$-norm over all row vectors.
\end{remark}

\begin{lemma}
	\label{lemma-1}
	For any $x \in \mathbb{R}^d$, the following inequality holds:
	$$||x|{|_\infty } \le ||x|{|_2} \le \sqrt d  \cdot ||x|{|_\infty }.$$
\end{lemma}
\begin{proof}	
	For the left-hand inequality, since	$||x||_2^2 = \sum\limits_{i = 1}^d {|{x_i}{|^2}}  \ge |{x_j}{|^2},\forall j$, it follows that $||x||_2 \ge \max_i |x_i| = ||x||_\infty$. For the right-hand inequality, since each component does not exceed $||x||_\infty$, we have:
	\begin{equation}
		\label{a}
		||x||_2^2 = \sum\limits_{i = 1}^d {|{x_i}{|^2}}  \le \sum\limits_{i = 1}^d {||x||_\infty ^2}  = d \cdot ||x||_\infty ^2 \Rightarrow ||x|{|_2} \le \sqrt d  \cdot ||x|{|_\infty }.
	\end{equation}
	
\end{proof}

\begin{lemma}
	\label{lemma0}
	For any $x \in \mathbb{R}^d$, assuming the dimension $d$ is fixed, the boundedness of the $\infty$-norm is equivalent to the boundedness of the Euclidean norm.
\end{lemma}
\begin{proof}	
	If $||x||_\infty \le C$, then by \textcolor{blue}{\cref{lemma-1}}, we have $||x||_2 \le \sqrt{d}\, C$, which implies the Euclidean norm is bounded. Conversely, if $||x||_2 \le C$, then by \textcolor{blue}{\cref{lemma-1}}, we have $||x||_\infty \le C$, which implies the $\infty$-norm is also bounded.
\end{proof}

\begin{lemma}
	\label{lemma1}
	The composition of two Lipschitz continuous functions is Lipschitz continuous. 
\end{lemma}
\begin{proof}
	Let $f:\mathbb{R}^n \to \mathbb{R}^m$ be Lipschitz continuous with constant $L_f$, satisfying $||f(x) - f(y)|| \le L_f ||x - y||$ for all $x, y \in \mathbb{R}^n$. Similarly, let $g:\mathbb{R}^m \to \mathbb{R}^k$ be Lipschitz continuous with constant $L_g$, satisfying $||g(u) - g(v)|| \le L_g ||u - v||$ for all $u, v \in \mathbb{R}^m$. For the composite function $h = g \circ f$, defined by $h(x) = g(f(x))$, we have
	\begin{equation}
		\label{b}
		||h(x) - h(y)|| = ||g(f(x)) - g(f(y))||
		\le L_g ||f(x) - f(y)|| \le L_g L_f ||x - y||,
	\end{equation}
	which shows that $h$ is Lipschitz continuous and its Lipschitz constant is $L_g L_f$.
\end{proof}

\begin{lemma}
	\label{lemma3}
	If every layer of an FNN has weight matrices with finite spectral norms and the activation function is Lipschitz continuous, then the entire FNN is Lipschitz continuous. 
\end{lemma}
\begin{remark}
	Note that this continuity holds for most common activation functions, such as ReLU, sigmoid, and tanh. Proofs for these functions are provided in the \textcolor{blue}{\cref{Proof of Lipschitz continuity for some activation functions}}.
\end{remark}
\begin{proof}	
	First, we show that the linear mapping $f(x) = Wx + b$ is Lipschitz continuous. For any $x, y$, we have
	\begin{equation}
		\label{c}
		||f(x) - f(y)|| = ||W(x - y)|| \le ||W|| \cdot ||x - y||,
	\end{equation}
	where $||W|| = ||W||_2$ denotes the spectral norm of the matrix $W$. Since $||W||$ is a finite constant, this linear transformation is Lipschitz continuous, and its Lipschitz constant is $||W||$. Next, consider an FNN consisting of $L$ layers, which can be written as
	\begin{equation}
		\label{d}
		f(x) = 
		\sigma(W_L \sigma(W_{L-1} \sigma(\cdots \sigma(W_1 x + b_1)\cdots ) + b_{L-1}) + b_L),
	\end{equation}
	
	where $\sigma$ is a Lipschitz continuous activation function. According to \textcolor{blue}{\cref{lemma1}}, the composition of a linear mapping and a Lipschitz continuous activation is also Lipschitz continuous. Therefore, each layer of the FNN is Lipschitz continuous. Since the entire FNN is a composition of these layers, \textcolor{blue}{\cref{lemma1}} implies that the FNN is Lipschitz continuous.
\end{proof}
\begin{lemma}
	\label{lemma4}
	Let $f, g: \mathbb{R}^n \to \mathbb{R}^m$ be bounded Lipschitz continuous functions, where $n$ and $m$ are fixed finite dimensions. Then the product function $h(x) = f(x) g(x)$ is Lipschitz continuous.
\end{lemma}
\begin{proof}
	Since $f(x)$ and $g(x)$ are Lipschitz continuous, let $L_f$ and $L_g$ denote their respective Lipschitz constants. Given that $f(x)$ and $g(x)$ are bounded, we assume $||f(x)||_\infty \le M$ and $||g(x)||_\infty \le M'$ for all $x$. By \textcolor{blue}{\cref{lemma0}}, it follows that $||f(x)|| \le M_1$ and $||g(x)|| \le M_2$, where $M_1 = \sqrt{n}\, M$ and $M_2 = \sqrt{n}\, M'$. Consider the function $h(x) = f(x) g(x)$. For any $x, y$, we obtain:
	\begin{multline}
		\label{e}
		||h(x) - h(y)|| = ||f(x) g(x) - f(y) g(y)||
		= ||f(x) g(x) - f(x) g(y) + f(x) g(y) - f(y) g(y)||\\
		\le ||f(x)|| \cdot ||g(x) - g(y)|| + ||g(y)|| \cdot ||f(x) - f(y)||\\
		\le M_1 \cdot L_g ||x - y|| + M_2 \cdot L_f ||x - y||
		= (M_1 L_g + M_2 L_f)\, ||x - y||,
	\end{multline}
	which satisfies the Lipschitz condition. Thus, $h(x)$ is Lipschitz continuous.
\end{proof}

\begin{lemma}
	\label{lemma5}	
	If the activation functions in an FNN are Lipschitz continuous and all weight matrices and biases are bounded, then for any bounded input $x \in \mathbb{R}^{d_0}$ (where $d_0$ is a fixed finite dimension), the output of the FNN is also bounded. 
\end{lemma}
\begin{proof}
	Consider an $L$-layer FNN defined by:
	\begin{equation}
		\label{f}
		{x^{(0)}} = x \in \mathbb{R}^{d_0}, 
		x^{(l)} = \sigma(W_l x^{(l-1)} + b_l), \; l = 1,2,\ldots,L.
	\end{equation}
	
	Since the activation function $\sigma$ is Lipschitz continuous, it satisfies $||\sigma(u) - \sigma(v)|| \le L_\sigma ||u - v||$ for all $u, v.$ In particular, setting $v = 0$ yields
	\begin{equation}
		\label{g}
		||\sigma(u) - \sigma(0)|| \le L_\sigma ||u|| 
		\Rightarrow	||\sigma(u)|| \le L_\sigma ||u|| + C_\sigma,
	\end{equation}
	where $C_\sigma = ||\sigma(0)||$ is a constant. Assume that all weight matrices $W_l$ and bias vectors $b_l$ are bounded, such that $||W_l||_\infty \le M_W$ and $||b_l||_\infty \le M_b.$ Furthermore, assume the input $x \in \mathbb{R}^{d_0}$ is bounded, satisfying $||x||_\infty \le M_x.$ We aim to show that the FNN output $x^{(L)}$ is bounded. That is, there exists a constant $C > 0$ such that
	\begin{equation}
		\label{h}
		||x^{(L)}||_\infty \le C, \quad  \forall x, ||x||_\infty \le M_x.
	\end{equation}
	
	We prove this by induction. For layer $0$, we have $||x^{(0)}||_\infty \le M_x$. For layer $1$, since $x^{(1)} = \sigma(W_1 x^{(0)} + b_1)$ and the entries satisfy $|W_{1,ij}| \le M_W$ and $|x^{(0)}_j| \le M_x$, each component satisfies
	\begin{equation}
		\label{i}
		| (W_1 x^{(0)})_i | 
		= \left| \sum_{j=1}^{d_0} W_{1,ij} x^{(0)}_j \right| 
		\le \sum_{j=1}^{d_0} |W_{1,ij}|\, |x^{(0)}_j|
		\le d_0 M_W M_x, \\[4pt]
		\Rightarrow | (W_1 x^{(0)} + b_1)_i | \le d_0 M_W M_x + M_b.
	\end{equation}
	
	Given that $||\sigma(u)|| \le L_\sigma ||u|| + C_\sigma$, \textcolor{blue}{\cref{lemma-1}} and \textcolor{blue}{\cref{lemma0}} imply that
	\begin{equation}
		\label{j}
		||\sigma(u)||_\infty \le ||\sigma(u)|| \le L_\sigma \sqrt{n}\, ||u||_\infty + C_\sigma.
	\end{equation}
	
	Therefore, we obtain
	\begin{equation}
		\label{k}
		|x^{(1)}_i|
		= |\sigma( (W_1 x^{(0)} + b_1)_i )|  
		\le L_\sigma \sqrt{n}\, |(W_1 x^{(0)} + b_1)_i| + C_\sigma,
	\end{equation}
	which leads to
	\begin{equation}
		\label{l}
		|x^{(1)}_i|
		\le L_\sigma \sqrt{n} (d_0 M_W M_x + M_b) + C_\sigma
		\;  \Rightarrow\; 
		||x^{(1)}||_\infty \le C_1,
	\end{equation}
	where $C_1 = L_\sigma \sqrt{n}(d_0 M_W M_x + M_b) + C_\sigma.$ Now, proceeding by induction, assume that the output of layer $l-1$ is bounded by $C_{l-1}$, so that $||x^{(l-1)}||_\infty \le C_{l-1}.$Layer $l$ is given by $x^{(l)} = \sigma(W_l x^{(l-1)} + b_l)$ and satisfies
	\begin{equation}
		\label{m}
		|(W_l x^{(l-1)} + b_l)_i| 
		\le d_{l-1} M_W C_{l-1} + M_b
		\; \Rightarrow\;
		|x^{(l)}_i|
		\le L_\sigma \sqrt{n}\, (d_{l-1} M_W C_{l-1} + M_b) + C_\sigma.
	\end{equation}
	
	Thus,
	\begin{equation}
		\label{n}
		||x^{(l)}||_\infty \le C_l,
	\end{equation}
	where $C_l = L_\sigma \sqrt{n}(d_{l-1} M_W C_{l-1} + M_b) + C_\sigma.$ This completes the proof.
\end{proof}

\begin{lemma}
	\label{lemma6}
	The product of two bounded functions $f, g: \mathbb{R}^n \to \mathbb{R}^m$ is also a bounded function, where $n$ and $m$ are fixed finite dimensions.
\end{lemma}
\begin{proof}
	Since the functions $f$ and $g$ are bounded, there exist constants $M_f, M_g > 0$ such that $||f(x)||_\infty \le M_f$ and $||g(x)||_\infty \le M_g$ for all $x$. Let $h(x) = f(x) g(x)$. Then, for any $x$, we have$$||h(x)||_\infty = ||f(x) g(x)||_\infty \le ||f(x)||_\infty \cdot ||g(x)||_\infty \le M_f M_g,$$which implies that $h(x)$ is also a bounded function.
\end{proof}
\begin{lemma}
	\label{lemma7}
	The sum of two Lipschitz continuous functions remains Lipschitz continuous.
\end{lemma}
\begin{proof}
	Suppose that the functions $f(x)$ and $g(x)$ are both Lipschitz continuous. Then there exist constants $L_f, L_g > 0$ such that for all $x, y$,
	\begin{align}
		||f(x) - f(y)|| \le L_f ||x - y||, \\
		||g(x) - g(y)|| \le L_g ||x - y||.
	\end{align}
	
	Define $h(x) = f(x) + g(x)$. For any $x, y$, we have
	\begin{multline}
		\label{p}
		||h(x) - h(y)|| 
		= ||f(x) + g(x) - f(y) - g(y)|| 
		= ||(f(x) - f(y)) + (g(x) - g(y))|| \\
		\le ||f(x) - f(y)|| + ||g(x) - g(y)|| 
		\le L_f ||x - y|| + L_g ||x - y|| 
		= (L_f + L_g) ||x - y||,
	\end{multline}
	which shows that $h(x)$ satisfies the Lipschitz condition. Therefore, $h(x)$ is Lipschitz continuous.
\end{proof}
\begin{theorem}[Picard–Lindelöf Theorem \citep{hartman2002ordinary}] 
	\label{Picard–Lindelöf Theorem}
	Let $y,f \in {\mathbb{R}^n};f(t,y)$ continuous on a parallelepiped $R: - a \le t - {t_0} \le a,||y - {y_0}|{|_\infty } \le b$  and Lipschitz continuous with respect to $y$. Let ${M_{||f||}}$  be a bound for $||f(t,y)|{|_\infty }$ on $R;\alpha  = \min \{ a,b/{M_{||f||}}\} $. Then ${y'} = f(t,y),y({t_0}) = {y_0}$	has a unique solution $y = y(t)$ on $[{t_0} - \alpha ,{t_0} + \alpha ]$.
\end{theorem}
This theorem establishes that the initial value problem admits solutions, provided that the right-hand side is Lipschitz continuous with respect to $y$ and continuous in $t$. The proof of \textcolor{blue}{\cref{Existence unique}} is provided below:
\begin{proof}	
	Let $z(t) = [z_1^t, z_2^t, \ldots, z_N^t] \in \mathbb{R}^{N \times d}$ denote the node hidden state at time $t$. The equations can be rewritten as a vector-valued differential equation of the form
	\begin{equation}
		\label{q}
		\frac{dz(t)}{dt} = \mathcal{F}(t, z(t)) := f^a(\mathcal{G}(t, z(t))) - z(t),
	\end{equation}
	where
	\begin{equation}
		\label{r}
		\mathcal{G}(t, z(t)) = \sum_{j \in V} \sum_{k=1}^K w_{i,j,k}^t \, f_k^r(z_i^t, z_j^t).
	\end{equation}
	
	Note that the explicit dependence on $t$ in the right-hand side arises only from the time-encoding component. Since the time-encoding function is continuous in $t$, the entire right-hand side is continuous with respect to $t$. Therefore, to guarantee the existence of the solution, it remains to show that $\mathcal{F}(t, z(t))$ is Lipschitz continuous with respect to $z(t)$. 
	
	For simplicity, we temporarily disregard the linear term $-z(t)$, as it can be absorbed into the preceding term $f^a(\mathcal{G}(t, z(t)))$ without affecting the final conclusion. We first analyze $f_k^r(z_i^t, z_j^t)$. In practice, this term is computed by concatenating $z_i^t$ and $z_j^t$ into a vector in $\mathbb{R}^{2d}$ and passing it through an FNN. Thus, it can be viewed as a function of the single variable $(z_i^t, z_j^t) \in \mathbb{R}^{2d}$, and its Lipschitz continuity depends on the corresponding FNN. Assume that all FNNs in the equations have bounded weights and biases and that their activation functions are Lipschitz continuous. According to \textcolor{blue}{\cref{lemma3}} and \textcolor{blue}{\cref{lemma5}}, such FNNs are Lipschitz continuous and bounded. Hence, $f_k^r(z_i^t, z_j^t)$ is Lipschitz continuous with bounded output. Similarly, since $\overline{w}_{i,k}^t$ and $\hat{w}_{i,k}^t$ are also outputs of FNNs, they are Lipschitz continuous and bounded. Moreover, \textcolor{blue}{\cref{lemma4}} and \textcolor{blue}{\cref{lemma6}} imply that their product $w_{i,j,k}^t$ remains Lipschitz continuous and bounded. Since both $w_{i,j,k}^t$ and $f_k^r(z_i^t, z_j^t)$ are Lipschitz continuous and bounded, \textcolor{blue}{\cref{lemma4}} ensures that their product is also Lipschitz continuous. Additionally, finite summation preserves Lipschitz continuity (\textcolor{blue}{\cref{lemma7}}), which implies that $\mathcal{G}(t, z(t))$ is Lipschitz continuous. Finally, as $f^a$ is implemented by an FNN, the property of composite functions (\textcolor{blue}{\cref{lemma1}}) confirms that $\mathcal{F}(t, z(t))$ is Lipschitz continuous with respect to $z(t)$. Consequently, the equation satisfies the conditions required in \textcolor{blue}{\cref{Picard–Lindelöf Theorem}}.
\end{proof}
\section{Proof of \textcolor{blue}{\cref{Robust}}}
\label{theorem2}

\begin{proof}
	Analogous to the proof of Theorem IV.1, we omit the term $- z_i^t$. Let $z_i^t$ denote the original system state, $\tilde z_i^t$ denote the perturbed system state, and the error be defined as
	\begin{equation}
		\label{s}
		e_i^t = \tilde z_i^t - z_i^t.
	\end{equation}
	
	The evolution of the perturbed system is given by:
	\begin{equation}
		\label{t}
		\frac{{d\tilde z_i^t}}{{dt}} = {f^a}(\sum\limits_j {\sum\limits_{k = 1}^K {\tilde w _{i,j,k}^t f_k^r(\tilde z _i^t,\tilde z _j^t)} } ) 
		\Rightarrow \frac{{de_i^t}}{{dt}} = \Delta _i^t: = {f^a}({\tilde A} _i) - {f^a}({A_i}),
	\end{equation}
	where
	\begin{align}
		\label{u}
		{ \tilde A _i} = \sum\limits_j {\sum\limits_{k = 1}^K {\tilde w _{i,j,k}^t f_k^r(\tilde z _i^t,\tilde z _j^t)} }, \\
		{A_i} = \sum\limits_j {\sum\limits_{k = 1}^K {w_{i,j,k}^tf_k^r(z_i^t,z_j^t)} }.
	\end{align}
	
	Given the Lyapunov function	$V(t) = \frac{1}{2} ||e^t||^2 = \frac{1}{2} \sum_i (e_i^t)^2,$ we obtain
	\begin{equation}
		\label{w}
		\frac{dV}{dt}
		= \sum_i \langle e_i^t, \Delta_i^t \rangle
		= \sum_i \sum_{l=1}^d e_{i,l}^t \cdot \Delta_{i,l}^t
		= \sum_i (e_i^t)^\top \Delta_i^t.
	\end{equation}
	
	By the Cauchy–Schwarz inequality, it follows that
	\begin{equation}
		\label{x}
		\frac{dV}{dt} \le \sum_i ||e_i^t|| \cdot ||\Delta_i^t||.
	\end{equation}
	
	Since the Lipschitz constant of $f^a$ is $L_a$, we can derive an upper bound for the perturbation term $||\Delta_i^t||$ as follows:
	\begin{equation}
		\label{y}
		||\Delta _i^t|| \le {L_a} \cdot ||{\tilde A _i} - {A_i}||.
	\end{equation}
	
	We further decompose this term as follows:
	\begin{multline}
		\label{z}
		||{{\tilde A }_i} - {A_i}|| \le \sum\limits_j {\sum\limits_{k = 1}^K {||\tilde w _{i,j,k}^tf_k^r(\tilde z _i^t,\tilde z _j^t) - w_{i,j,k}^tf_k^r(z_i^t,z_j^t)||} } \\
		\le \sum\limits_j {\sum\limits_{k = 1}^K {(||\tilde w _{i,j,k}^t - w_{i,j,k}^t|| \cdot ||f_k^r(\tilde z _i^t,\tilde z _j^t)||}} 
		+ ||w_{i,j,k}^t|| \cdot ||f_k^r(\tilde z _i^t,\tilde z _j^t) - f_k^r(z_i^t,z_j^t)||) .
	\end{multline}
	
	Given that the Lipschitz constants of $f_k^r$ and $w_{i,j,k}^t$ are $L_{r,k}$ and $L_{w,k}$, respectively, we obtain:
	\begin{multline}
		\label{aa}
		||f_k^r(\tilde z _i^t,\tilde z _j^t) - f_k^r(z_i^t,z_j^t)|| 
		\le {L_{r,k}}||[\tilde z _i^t - z_i^t,\tilde z _j^t - z_j^t]|| 
		= {L_{r,k}}\sqrt {||\tilde z _i^t - z_i^t|{|^2} + ||\tilde z _j^t - z_j^t|{|^2}} \\
		\le {L_{r,k}}(||\tilde z _i^t - z_i^t|| + ||\tilde z _j^t - z_j^t||) 
		\le {L_{r,k}}(||e_i^t|| + ||e_j^t||),
	\end{multline}
	\begin{multline}
		\label{ab}
		||\tilde w _{i,j,k}^t - w_{i,j,k}^t|| 
		= ||w_{i.j,k}^t(\tilde z _i^t,\tilde z _j^t) - w_{i.j,k}^t(z_i^t,z_j^t)||
		= ||\overline w_{i,k}^t(\tilde z _i^t) \cdot \hat w_{j,k}^t(\tilde z _j^t) - \overline w_{i,k}^t(z_i^t) \cdot \hat w_{j,k}^t(z_j^t)||\\
		= ||\overline w_{i,k}^t(\tilde z _i^t) \cdot \hat w_{j,k}^t(\tilde z _j^t) - \overline w_{i,k}^t(z_i^t) \cdot \hat w_{j,k}^t(\tilde z _j^t) 
		+ \overline w_{i,k}^t(z_i^t) \cdot \hat w_{j,k}^t(\tilde z _j^t) - \overline w_{i,k}^t(z_i^t) \cdot \hat w_{j,k}^t(z_j^t)||\\
		\le ||\hat w_{j,k}^t(\tilde z _j^t)|| \cdot ||\overline w_{i,k}^t(\tilde z _i^t) - \overline w_{i,k}^t(z_i^t)|| 
		+ ||\overline w_{i,k}^t(z_i^t)|| \cdot ||\hat w_{j,k}^t(\tilde z _j^t) - \hat w_{j,k}^t(z_j^t)||\\
		\le {L_{w,k}}||e_i^t|| \cdot ||\hat w_{j,k}^t(\tilde z _j^t)|| + {L_{w,k}}||e_j^t|| \cdot ||\overline w_{i,k}^t(z_i^t)||
		\le {L_{w,k}}\sqrt d (||e_i^t|| + ||e_j^t||).
	\end{multline}
	
	The inequality $L_{w,k} ||e_i^t|| \cdot \bigl\|\hat w_{j,k}^t(\tilde z_j^t)\bigr\|
	\;+\;
	L_{w,k} ||e_j^t|| \cdot \bigl\|\overline w_{i,k}^t(z_i^t)\bigr\|
	\;\le\;
	L_{w,k} \sqrt{d}\,\bigl(||e_i^t|| + ||e_j^t||\bigr)$
	is valid because $\sigma$ is the sigmoid activation function. Consequently, $\bigl\|\hat w_{j,k}^t(\tilde z_j^t)\bigr\|
	\le \sqrt{d}\,\bigl\|\hat w_{j,k}^t(\tilde z_j^t)\bigr\|_\infty
	\le \sqrt{d},$ and similarly, $\bigl\|\overline w_{i,k}^t(z_i^t)\bigr\|
	\le \sqrt{d}\,\bigl\|\overline w_{i,k}^t(z_i^t)\bigr\|_\infty
	\le \sqrt{d}.$ Furthermore, since $f_k^r(z_i^t, z_j^t)$ and $w_{i,j,k}^t$ are bounded for all $k$, \textcolor{blue}{\cref{lemma0}} ensures their Euclidean norms are bounded by $\sqrt{d}\, C_{r,k}$ and $\sqrt{d}\, C_{w,k}$. This yields the perturbation bound:
	\begin{equation}
		\label{ac}
		||\Delta_i^t|| 
		\le \sum_j \sum_{k=1}^KL_a \bigl(d C_{r,k} L_{w,k} + \sqrt{d}\, C_{w,k} L_{r,k}\bigr)\bigl(||e_i^t|| + ||e_j^t||\bigr) \
		\le \sum_{k=1}^K \alpha_k ||e^t||,
	\end{equation}
	where $\alpha_k = 2N \cdot L_a \bigl(d C_{r,k} L_{w,k} + \sqrt{d}\, C_{w,k} L_{r,k}\bigr).$
	Inserting this into the Lyapunov derivative, we obtain
	\begin{equation}
		\label{ad}
		\frac{dV}{dt}
		\le \sum_i ||e_i^t|| \cdot ||\Delta_i^t||
		\le ||e^t|| \sum_i ||e_i^t|| \cdot \sum_{k=1}^K \alpha_k.
	\end{equation}
	
	Applying the Cauchy–Schwarz inequality again gives
	\begin{equation}
		\label{ae}
		\frac{dV}{dt}
		\le ||e^t||^2 \cdot \sqrt{N \Bigl(\sum_{k=1}^K \alpha_k\Bigr)^2}
		= ||e^t||^2 \sqrt{N} \sum_{k=1}^K \alpha_k.
	\end{equation}
	
	Similarly, for the mean-interaction system,
	\begin{equation}
		\label{af}
		\frac{dV_{\mathrm{average}}}{dt}
		\le ||e^t||^2 \sqrt{N \alpha^2}
		= ||e^t||^2 \sqrt{N}\, \alpha.
	\end{equation}
	
	As $L^k = d C_{r,k} L_{w,k} + \sqrt{d}\, C_{w,k} L_{r,k}$ and $\sum_k L^k \le L_r$ by assumption, it implies that $\sum_{k=1}^K \alpha_k \le \alpha.$
\end{proof}

\section{Proof of Lipschitz continuity for some activation functions}
\label{Proof of Lipschitz continuity for some activation functions}
\begin{lemma}\citep{lebl2009basic}
	\label{lemma2}
	If a function $f$ is differentiable and its derivative is bounded (specifically, there exists $M > 0$ such that $|f'(x)| \le M$ for all $x \in \mathbb{R}$), then $f$ is Lipschitz continuous, and its Lipschitz constant is at most $M$.
\end{lemma}
Next, we provide proofs that several commonly used activation functions are Lipschitz continuous:
\begin{itemize}
	\item ReLU: $f(x) = \max(0,x)$
	\begin{proof}
		For any $x, y \in \mathbb{R}^d$, he inequality
		\begin{equation}
			\label{ag}
			|\max(0, x_i) - \max(0, y_i)| \le |x_i - y_i|,
		\end{equation}
		holds for each component. Hence, we have
		\begin{equation}
			\label{ah}
			||\mathrm{ReLU}(x) - \mathrm{ReLU}(y)||^2 
			= \sum_{i=1}^d |\max(0, x_i) - \max(0, y_i)|^2 
			\le \sum_{i=1}^d |x_i - y_i|^2 
			= ||x - y||^2.
		\end{equation}
		
		Taking the square root on both sides yields
		\begin{equation}
			\label{ai}
			||\mathrm{ReLU}(x) - \mathrm{ReLU}(y)|| \le ||x - y||.
		\end{equation}
		
		Therefore, ReLU is Lipschitz continuous.
	\end{proof}
	\item Sigmoid: $\sigma(x) = \frac{1}{1 + e^{-x}}$
	\begin{proof}
		The derivative of the scalar function $\sigma(x)$ is given by
		\begin{equation}
			\label{aj}
			\sigma'(x) = \sigma(x)\bigl(1 - \sigma(x)\bigr).
		\end{equation}
		
		Since $0 < \sigma(x) < 1$ for all $x \in \mathbb{R}$, it follows that $\sigma'(x) \in (0, 0.25]$. Thus, for any $x, y \in \mathbb{R}^d$, \textcolor{blue}{\cref{lemma2}} implies that
		\begin{equation}
			\label{ak}
			|\sigma(x_i) - \sigma(y_i)| \le 0.25\, |x_i - y_i|.
		\end{equation}
		
		Consequently, we have
		\begin{equation}
			\label{al}
			||\sigma(x) - \sigma(y)||^2 
			= \sum_{i=1}^d |\sigma(x_i) - \sigma(y_i)|^2 
			\le 0.25^2 \sum_{i=1}^d |x_i - y_i|^2
			= 0.25^2\, ||x - y||^2.
		\end{equation}
		
		Taking the square root of both sides yields
		\begin{equation}
			\label{am}
			||\sigma(x) - \sigma(y)|| \le 0.25\, ||x - y||.
		\end{equation}
		
		Therefore, the sigmoid function is Lipschitz continuous.
	\end{proof}
	\item Tanh: $f(x) = \tanh(x)$
	\begin{proof}
		The derivative of the scalar function $\tanh(x)$ is given by
		\begin{equation}
			\label{an}
			\tanh'(x) = 1 - \tanh^2(x).
		\end{equation}
		
		Since $-1 < \tanh(x) < 1$, it follows that $0 < \tanh'(x) \le 1$, where equality holds if and only if $\tanh(x) = 0$. Thus, for any $x, y \in \mathbb{R}^d$, \textcolor{blue}{\cref{lemma2}} implies
		\begin{equation}
			\label{ao}
			|\tanh(x_i) - \tanh(y_i)| \le |x_i - y_i|.
		\end{equation}
		
		Consequently, we have
		\begin{equation}
			\label{ap}
			||\tanh(x) - \tanh(y)||^2
			= \sum_{i=1}^d |\tanh(x_i) - \tanh(y_i)|^2 
			\le \sum_{i=1}^d |x_i - y_i|^2
			= ||x - y||^2.
		\end{equation}
		
		Taking the square root of both sides yields
		\begin{equation}
			\label{aq}
			||\tanh(x) - \tanh(y)|| \le ||x - y||.
		\end{equation}
		
		Therefore, the $\tanh$ function is Lipschitz continuous.
	\end{proof}
\end{itemize}
\printcredits

\bibliographystyle{cas-model2-names}

\bibliography{cas-refs}


\end{document}